\documentclass[11pt]{article}
\usepackage[small,compact]{titlesec}
\usepackage{amsmath,amssymb,amsfonts,mathabx,setspace,bm} %% bm for using \bm{rho} to get bold symbols; 
\usepackage{graphicx,caption,epsfig,subfigure,epsfig,wrapfig}
\usepackage{url,verbatim,algorithmic}
\usepackage[table]{xcolor}
\DeclareRobustCommand{\best}[1]{{\color{green!40!black}\ifmmode\boldsymbol{#1}\else\bfseries\boldmath #1\fi}}
\usepackage[sort,compress]{cite}
\usepackage{multirow}
\usepackage{enumitem}
\setlist[itemize]{leftmargin=1.5em}
\usepackage{booktabs}
\usepackage{longtable}
\usepackage{pdflscape}
\usepackage[section]{placeins}  % keep floats within their section
\usepackage[ruled,boxed]{algorithm}
\usepackage[scaled]{helvet}
\usepackage[T1]{fontenc}
\usepackage[bookmarks=true, bookmarksnumbered=true, colorlinks=true,   pdfstartview=FitV,
linkcolor=blue, citecolor=blue, urlcolor=blue]{hyperref}

\usepackage{fullpage}
\def\R{\mathbb{R}}

\newcommand{\E}[1]{\mathbb{E}\left[{#1}\right]}

\newcommand{\innerp}[1]{\langle{#1}\rangle}

\def\calE{\mathcal{E}}
\def\bX{\mathbf{X}}        

\def\bW{\mathbf{W}}

\definecolor{dgreen}{RGB}{0,153,76}

\newtheorem{theorem}{Theorem}

\newtheorem{assumption}[theorem]{Assumption}

\newtheorem{corollary}[theorem]{Corollary}

\newtheorem{lemma}[theorem]{Lemma}
\newtheorem{proposition}[theorem]{Proposition}
\newtheorem{remark}[theorem]{Remark}
\newenvironment{proof}[1][Proof]{\noindent\textbf{#1.} }{\ \rule{0.5em}{0.5em}}

\numberwithin{equation}{section} % equation number within section
\numberwithin{theorem}{section}  % theorem number within section

\begin{document}
\title{\Large Learning interacting particle systems from unlabeled data}

\author{
Viska Wei\thanks{Department of Applied Mathematics and Statistics and Department of Physics and Astronomy, Johns Hopkins University, Baltimore, MD 21218, USA. Email: \texttt{swei20@jhu.edu}} \and
Fei Lu\thanks{Department of Mathematics, Johns Hopkins University, Baltimore, MD 21218, USA. Email: \texttt{feilu@math.jhu.edu}}
}

\date{}
\maketitle
\thispagestyle{empty}
\vspace{-8mm}
\begin{abstract}
Learning the potentials of interacting particle systems is a fundamental task across various scientific disciplines. A major challenge is that unlabeled data collected at discrete time points lack trajectory information due to limitations in data collection methods or privacy constraints. We address this challenge by introducing a trajectory-free self-test loss function that leverages the weak-form stochastic evolution equation of the empirical distribution. The loss function is quadratic in potentials, supporting parametric and nonparametric regression algorithms for robust estimation that scale to large, high-dimensional systems with big data. Systematic numerical tests show that our method outperforms baseline methods that regress on trajectories recovered via label matching, tolerating large observation time steps. We establish the convergence of parametric estimators as the sample size increases, providing a theoretical foundation for the proposed approach.
\end{abstract}

%% set depth for table of contents
 \setcounter{tocdepth}{1}
 \tableofcontents

%%%%%%%%% ===========================
\section{Introduction}

Learning the dynamics of interacting particle systems is a central problem across physics, biology, social science, and neuroscience \cite{howard2022learning,acemoglu2011opinion,MT14,sharrock2021_ParameterEstimation,olfati2007consensus,jenkinson2017potential,bjornsson2008intra,razin1991dna,delarue2015particle}. A central inverse task is to recover the interaction and external potentials from data. In many applications, however, the data are unlabeled particle snapshots at discrete times that miss the trajectory information of the particles across time. This loss of particle labels and trajectory information, caused by imaging limitations or privacy constraints, makes the dynamics inference substantially difficult. 

Specifically, consider the estimation of the interaction potential $\Phi$ and external potential $V$ in the interacting particle system of $N$ particles in $\R^d$: 
\begin{equation}\label{eq:opt_model_R} 
	d X_t^i = -\frac{1}{N} \sum_{j=1, j\neq i}^N \nabla \Phi(X_t^i-X_t^j) dt
	%\nolimits_{j\neq i} \phi(|X_i-X_j|) \frac{X_i-X_j}{|X_i-X_j|}dt 
	 -\nabla V(X_t^i) dt + \sigma dW_t^i, \quad i= 1,\ldots\, N\,, \quad t\in [0,T]
\end{equation}
from \emph{unlabeled data} consisting of a sequence of sample ensembles denoted by
\begin{equation}\label{eq:data_ensemble}
	\mathcal{D} = \{\bX_{t_\ell}^{\pi,m}\}_{\ell,m=1}^{L,M} \text{ with } \bX_t^{\pi} = (X_t^{\pi_t(1)},\ldots, X_t^{\pi_t(N)}), 
\end{equation}
Here, the stochastic noise $W_t^i \in \R^d$ is a standard Brownian motion, $\sigma\geq 0$ is the noise intensity, $\{t_\ell\}_{\ell=0}^L\subset [0,T]$ are time points with $t_\ell = \ell \Delta t$, $M$ is the number of samples, and $\pi_t$ is an unknown permutation of particle indices $\{1,\ldots,N\}$ at time $t$. In particular, the data is unlabeled in the sense that the permutation $\pi_t$ is unknown for all $t$. This means that the data misses trajectory information, i.e., the positions $X_{t_\ell}^{i,m}$ and $X_{t_{\ell+1}}^{i,m}$ are not paired in data since the label is unknown. 

% \begin{figure}[ht]
% \centering
% \includegraphics[width=0.85\textwidth]{figures/problem_setup.pdf}
% \caption{Labeled snapshots (left) pair each particle across time via known indices; unlabeled snapshots (right) lose this correspondence. }
% \label{fig:problem_setup}
% \end{figure}

The unlabeled setting breaks the main assumption behind most existing estimators based on particle trajectories, such as velocity matching, likelihood maximization, or Bayesian trajectory inference (see e.g., \cite{Kut04,DS04,Iac09,LMT21_JMLR,LZTM19pnas,LMT21,kasonga1990_MaximumLikelihooda,liu2020_ParameterEstimation,amorino2023parameter,huang2019learning,chen2021_MaximumLikelihooda,wen2016_MaximumLikelihood,FRT21-GP}). Moreover,  the number of particles is often not large enough to use the mean-field equation for the inference \cite{bongini2017inferring,LangLu22,wen2016_MaximumLikelihood,messenger2022learning}. As a result, learning from unlabeled data is a major open problem in the field, and existing methods either resort to label (and hence trajectory) recovery \cite{schiebinger2019optimal,liang2023recurrent,chapel2020partial} or minimization of distributional distance between the data and model predictions \cite{li2024stochastic,li2024stochastic}, both of which are computationally expensive and can be inaccurate when the time gap between observations is large. 

In this work, we overcome the challenge by introducing a \emph{self-test loss function} based on the weak-form partial differential equation (PDE) for the evolution of the empirical distribution of the particles.  This approach is inspired by the self-test framework in \cite{gao2024self}. Specifically, by It\^{o}'s chain rule, the empirical distribution of the particles 
\begin{equation}\label{eq:empirical_measure}
\mu_t^N (x):= \mu^{\bX_t} (x) = \frac{1}{N}\sum_{i=1}^N \delta_{X_t^i}(x)
\end{equation}
satisfies weak-form (in the distributional sense) stochastic evolution equation:
\begin{equation*}
\partial_t \mu_{t}^N =   \nabla \cdot   \left[ \mu_{t}^N \nabla  \big( \Phi*\mu_{t}^N + V\big)  \right]   +\frac{\sigma^2}{2} \Delta \mu_{t}^N + \dot m^{\bX_t},  \quad x\in \R^d, t\in [0,T]
\end{equation*}
with $\dot m^{\bX_t}$ being a martingale noise. The task of learning $\Phi$ and $V$ from the unlabeled data is equivalent to estimating them in the weak-form PDE from the empirical distributions $\{\mu_{t_\ell}^{N,m}\}$ of the data. By using $\{\Phi*\mu_{t_\ell}^{N,m} + V\}$ as test functions, we obtain a trajectory-free loss function for learning $\Phi$ and $V$ from the data. The loss is \emph{quadratic} in the potentials, and each of its minimizer satisfies the weak-form PDE up to a martingale term.

 The resulting objective is quadratic in the potentials, supports both parametric and neural-network estimators, and remains effective at coarse observation intervals where label recovery becomes unreliable. It avoids the computational expense and inaccuracy of label recovery and distributional matching methods, providing a robust and scalable solution for learning from unlabeled data in interacting particle systems.

\paragraph{Contributions.} Our main contributions are:
\begin{itemize}
\item We introduce a trajectory-free loss function for learning the interaction and external potentials from unlabeled ensemble data. The loss is quadratic in the potentials and is based on the weak-form PDE of the empirical particle distribution (Section~\ref{sec:loss_fn}). 
\item We present both parametric and neural network estimators for minimizing the loss function (Section~\ref{sec:computation}). In particular, we prove error bounds for the parametric estimator, showing convergence of the estimator as the sample size increases and the observation interval decreases (Section~\ref{sec:theory}).
\item We systematically examine the performance of the proposed algorithms on synthetic data generated from six models, including a reference model with regularity assumptions and five stress-test radial and non-radial models. We validate the convergence of the parametric estimator on the reference model. We also compare our method with baseline methods that regress on trajectories recovered via label matching, showing that our method outperforms them at large time steps where label recovery fails, scales with sample size and time step size, and can recover non-radial potentials (Section~\ref{sec:num}). 
\end{itemize}

%%%%%%%%%%%%%%%%%%% ================
\subsection{Related work}
There is relatively limited prior work on learning the potentials from unlabeled data for interacting particle systems with finitely many particles. 
% Most existing methods either attempt to recover the missing labels and then apply trajectory-based methods, or minimize the distances between distributions of the data and the model prediction. We review these approaches below, along with related literature on inference for IPS and mean-field limits, and weak form methods for equation learning, which are relevant to our approach. We review the related literature in three parts: inference for IPS and mean-field limits, and weak form methods for equation learning,

\paragraph{Generic methods based on distribution transport.} The generic approaches mainly fall into two categories: regression after label recovery and distributional matching. Since the particles are unlabeled, one may attempt regression after recovery of the labels (equivalently, trajectory reconstruction) by optimal transport methods \cite{genevay2016stochastic,bunne2022proximal,li2025robust}. However, as we show in the numerical section, this approach is computationally expensive and can be inaccurate when the time gap between observations is large. Another approach is to directly minimize a distributional distance (e.g., Wasserstein distance) between the empirical distribution of the data and the distribution predicted by the model \cite{yang2022generative,li2024stochastic,li2025inverse}. However, this approach is also computationally expensive because it requires simulating the full particle system at every training step and solving an optimal transport problem to compute the loss gradient. 

\paragraph{Mean-field based methods.} This study applies the self-testing functions in \cite{gao2024self} to construct a loss function, using the empirical distribution's weak-form PDE that resembles the mean-field equation up to a martingale-induced noise term. The pioneering work \cite{messenger2022learning} uses the weak form of the mean-field PDE to construct a loss function that can be estimated from the unlabeled data. However, the test functions are a library to be pre-selected. Another pioneering work \cite{bongini2017inferring} learns the interaction kernel in the mean-field limit using the derivative of the particle positions, which requires particle labels. Recent work on learning for mean-field equations includes \cite{chen2021_MaximumLikelihooda, LangLu22, yao2022mean,bongini2017inferring}, but none of these methods can be applied to the weak-form PDE of the empirical distribution of finitely many particles.

\paragraph{Energy variational methods.} The energy-dissipation-based loss in \cite{luli2024learningEV} aims to learn stochastic differential equations from density data by enforcing the energy balance of the system, which is a fundamental energy variational principle \cite{wang2021particle}. Applying it directly to our setting, the energy balance will lead to a loss function that is quartic in the potentials, which is more difficult to optimize than our quadratic loss function. The central difference is that the self-test loss function is based on variation over the potentials, while the energy balance is based on variation over the density.

The paper is organized as follows. Section~\ref{sec:loss_fn} introduces the trajectory-free loss function and summarizes its key properties. Section~\ref{sec:computation} presents parametric and neural network algorithms for minimizing the loss function. Section~\ref{sec:theory} establishes error bounds for the parametric estimator, showing convergence as the sample size increases and the observation interval decreases. Section~\ref{sec:num} systematically tests the method on synthetic data from six models, comparing it with baseline methods based on label recovery. Finally, Section~\ref{sec:conclusion} concludes with a discussion of future directions. The code is available at \url{https://github.com/ViskaWei/lips_unlabeled_data}.
% \href{https://github.com/ViskaWei/lips_unlabeled_data}{GitHub}. 

%%%%%%%%%%%%%%%% ============== 
%%%%%%%%%%%%%%%% ============== 
%%%%%%%%%%%%%%%% ============== 

\section{The trajectory-free self-test loss}\label{sec:loss_fn}
This section introduces a trajectory-free loss function for learning the potentials from unlabeled ensemble data. Section~\ref{sec:loss_fn_formulation} presents the loss function and summarizes its key properties; Section~\ref{sec:derivation_loss_fn} derives it from the weak-form PDE of the empirical distribution; Section~\ref{sec:discretization} analyzes the discretization bias when the loss is evaluated on discrete data.
\subsection{Loss formulation and key properties} \label{sec:loss_fn_formulation}
We propose a new trajectory-free loss function for learning the potential functions from the unlabeled ensemble data \eqref{eq:data_ensemble}. It is based on a weak-form stochastic PDE satisfied by the empirical distribution of the system and it is inspired by the self-test loss function in \cite{gao2024self}.

The weak-form PDE is from the It\^{o}  chain rule. Recall that for any test function $f\in C^2_b(\R^d)$, It\^{o}'s formula leads to:
 \begin{equation*}
d \frac{1}{N} \sum_i f(X_t^i)    =  \frac{1}{N} \sum_i  \big[ -\nabla f(X_t^i) \cdot \big(\nabla V(X_t^i) + \frac{1}{N} \sum_j \nabla \Phi(X_t^i-X_t^j) \big)  + \frac{\sigma^2}{2}   \Delta f(X_t^i) \big] dt + \sigma d m^{\bX_t}(f),
 \end{equation*}
with $d m^{\bX_t}(f)$ being the differential of the martingale
\begin{equation}\label{eq:martingale}
m^{\bX_t}(f)  :=\int_0^t \frac{1}{N} \sum_i  \nabla f(X_t^i)  dW_t^i.
\end{equation}
Equivalently, in terms of the empirical distribution $\mu_t^N$ in \eqref{eq:empirical_measure}, we have
\begin{equation*}
\begin{aligned}
d \innerp{\mu_t^N, f}  = & \left(  \innerp{\nabla. [\mu_t^N (\nabla \Phi * \mu_t^N + \nabla V)], f} + \innerp{ \frac{\sigma^2}{2} \Delta \mu_t^N, f} \right) dt  + \sigma d m^{\bX_t}(f).
\end{aligned}
\end{equation*}
That is, $\mu_t^N$ evolves according to the following weak-form stochastic evolution equation (interpreted in the distributional sense):
\begin{equation}\label{eq:dist-PDE}
\partial_t \mu_{t}^N =  \nabla \cdot   \left[ \mu_{t}^N \nabla  \big( \Phi*\mu_{t}^N + V\big)  \right]   +\frac{\sigma^2}{2} \Delta \mu_{t}^N + \sigma \dot{m}^{\bX_t}, \quad x\in \R^d, t\in [0,T],
\end{equation}
where $\dot{m}^{\bX_t}$ is the time-derivative of the martingale. Since this martingale has mean zero and variance $\mathrm{Var}(m^{\bX_t}(f)) = O(1/N)$ by It\^o's isometry, it can be viewed as a small observation noise, and the error caused by it vanishes as the sample size $M$ and the number of particles $N$ increases.

Note that Eq.\eqref{eq:dist-PDE} is not the Liouville or Fokker-Planck equation of the ODE or SDE, which characterizes the evolution of the joint probability density of the particles on $\R^{Nd}$. Instead, it is a PDE for the evolution of the empirical distribution of the particles on $\R^d$. In particular, when $\sigma = 0$ it reduces to a transport equation; when $\sigma > 0$ it is a non-closed stochastic PDE with a state-dependent martingale term $m^{\bX_t}(f)$. 

Importantly, the empirical distribution $\mu_t^N$ is observable from the unlabeled snapshot data, and the weak-form PDE \eqref{eq:dist-PDE} characterizes how the evolution of $\mu_t^N$ depends on the potentials without requiring individual trajectories. Thus, it provides the foundation for constructing a trajectory-free loss function. 

We construct a trajectory-free loss function via the self-testing approach for weak-form PDEs in \cite{gao2024self}. The key idea is to test the PDE against a family of test functions $f = V + \Phi * \mu_t^N$ with $\mu_t^N\in \{\mu_{t_\ell}^{N,m}\}_{\ell=0,m=1}^{L-1,M}$, which depend on the potentials themselves and the data, and construct a loss function whose Fr\'echet derivative at the true potentials is zero so that the weak-form PDE for the observed $\mu_t^N$ is satisfied up to the unobservable zero-mean martingale term. See Section~\ref{sec:derivation_loss_fn} for a detailed derivation. 

The trajectory-free self-test loss function for learning $\Phi$ and $V$ from the data is given by
\begin{equation}\label{eq:loss_trajFree}
\begin{aligned}
\calE_{\mathcal{D}}(\Phi,V) 	 =&  \frac{1}{MT}\sum_{m,\ell=1}^{M,L}  \calE_{\bX_{t_\ell}^{m},\bX_{t_{\ell+1}}^{m}}(\Phi,V), \quad \text{ with } \\
\calE_{\bX_{t_\ell},\bX_{t_{\ell+1}}}(\Phi,V)
 = & \frac{1}{2}\underbrace{\int |\nabla V+ \nabla \Phi *\mu_{t_\ell}^N|^2\mu_{t_\ell}^N dx  }_{J_{diss}: \text{Dissipation}}  \Delta t
  - \frac{\sigma^2}{2}\underbrace{ \int [\Delta V+ \Delta \Phi *\mu_{t_\ell}^N ] \mu_{t_\ell}^N dx }_{J_{diff}: \text{Diffusion}} \Delta t \\
  & + \underbrace{\int [ V+  \frac{1}{2} \Phi *\mu_{t}^N ] \mu_{t}^N dx \Big|_{t_\ell}^{t_{\ell+1}}}_{\delta E_{f}: \text{Energy~change}},
\end{aligned}
\end{equation}
where by the definition of $\mu_t^N$, the integrals can be computed as:
\begin{equation}\label{eq:loss_terms}
\begin{aligned}
J_{diss}(\bX_t)  & :=   \frac{1}{N} \sum_{i = 1}^N\big|\nabla V(X_{t}^i)+ \frac{1}{N}\sum_{j \neq i}^N \nabla \Phi (X_{t}^i- X_{t}^j) \big|^2, \\
J_{diff}(\bX_t)  & := \frac{1}{N} \sum_{i = 1}^N \Delta V(X_t^i) + \frac{1}{N^2} \sum_{i\neq j} \Delta \Phi(X_t^i - X_t^j), \\
E_{f}(\bX_t)   & := \frac{1}{N}\sum_{i=1}^N V(X_{t}^i) + \frac{1}{2N^2}\sum_{i,j= 1, i\neq j}^N \Phi (X_{t}^i- X_{t}^j), \quad \delta E_{f} = E_{f}(\bX_{t_{\ell+1}}) - E_{f}(\bX_{t_\ell}).
\end{aligned}
\end{equation}
Here, the first term $J_{diss}$ represents the energy dissipation due to the drift, the second term $J_{diff}$ represents the contribution from diffusion, and the third term $\delta E_{f}$ represents the change in the free energy of the system between two time points.
For even $\Phi$, $\nabla \Phi(0)=0$, so excluding the self-interaction $j=i$ in $J_{diss}$ is equivalent to including it. We write both pairwise sums with $j\neq i$ to match the finite-particle free energy $E_f$ and the regression vectors used later.

% As we will see in the derivation of the loss function in the next section, the minimizer of the loss function conserves the energy balance of the system, which is a fundamental physical principle.
The self-test loss function has several important properties.
\begin{itemize}
    \item \emph{Trajectory-free and derivative-free}: all three terms $J_{diss}$, $J_{diff}$, and $E_{f}$ depend on the particle positions only through the empirical distribution $\mu_t^N$, so no particle labels or trajectories are needed. Also, it does not require any derivative of $\mu_t^N$, making it suitable for empirical distribution data with relatively large time steps. % In contrast, the mean-square prediction error of the drift or the energy balance requires approximating the time derivative of the particle positions, which is not possible without labels and is inaccurate for large time steps.
   \item \emph{Quadratic}: the loss function is quadratic in the potentials, which is different from the Wasserstein-distance-based loss functions that are typically non-convex and highly nonlinear in the potentials, or the mean-square prediction error of the energy balance that is quartic (see Section~\ref{sec:derivation_loss_fn}). This quadratic structure allows for efficient optimization and robust convergence properties. In particular, it leads to a least squares regression problem when the potentials are linearly parameterized. 
\item \emph{Minimizers well-characterized}: At each minimizer, the Fr\'echet derivative of the loss function is zero, which leads to the weak-form PDE \eqref{eq:dist-PDE} without the martingale term. Thus, each minimizer makes all the data $\{\mu_t^{N,m}\}$ satisfy the weak-form PDE up to a martingale term. 
Moreover, under a coercivity condition, the minimizer is unique and hence the true potentials are identifiable from data; see Theorem~\ref{thm:error_bound} for an error bound.
\end{itemize}

We therefore define the estimator as the minimizer of the loss over a hypothesis space $\mathcal{H} = \mathcal{H}_V \times \mathcal{H}_\Phi$ that enforces the required symmetry and smoothness conditions:
\begin{equation}\label{eq:estimator}
(\widehat{\Phi}, \widehat{V}) = \arg \min_{(\Phi,V)\in \mathcal{H}} \calE_{\mathcal{D}}(\Phi,V).
\end{equation}
The spaces $\mathcal{H}_\Phi$ and $\mathcal{H}_V$ can be chosen as parametric classes (e.g., a linear space with pre-selected basis functions) or nonparametric function spaces (e.g., RKHS or neural networks). In particular, we require that $\Phi$ is even: $\Phi(x) = \Phi(-x)$ for all $x \in \R^d$, which is necessary for the drift to be the gradient of the free energy $E_{f}$. We also assume that the potentials are smooth enough such that the gradients and Laplacians in \eqref{eq:loss_terms} are well defined. Finally, the loss is invariant under additive shifts of $V$ and $\Phi$, so the potentials are identifiable only up to additive constants. In practice we do not need to enforce these constants during optimization; when a canonical representative is needed, one may impose $V(0)=0$ and $\Phi(0)=0$.

%%%%%%%%%
\subsection{Derivation of the loss function}\label{sec:derivation_loss_fn}
We derive the trajectory-free loss function \eqref{eq:loss_trajFree} using the self-test approach for weak-form PDEs in \cite{gao2024self}. The key structural property that enables the self-test approach is that the weak-form PDE depends linearly on the unknown potentials, despite being nonlinear in the empirical distribution. This linearity allows us to construct a quadratic loss function whose minimizer makes the data sastisfy the weak-form PDE. 

To highlight the key ideas, we first reformulate the problem in an abstract weak operator form. Here, we present a formal derivation of the loss function, ignoring technical details such as regularity and integrability conditions. The rigorous justification of the loss function and the convergence of the estimator are given in Section~\ref{sec:theory}. 

Let $\phi = (V, \Phi)$ denote the unknown potentials. Define the \emph{drift operator}
\begin{equation}\label{eq:operator_R}
    R_{\phi}[u] := -\nabla \cdot \big[u \,(\nabla V + \nabla \Phi * u)\big],
\end{equation}
which maps a probability measure $u$ to the negative divergence of the drift flux. The weak-form PDE~\eqref{eq:dist-PDE} can then be rewritten as a weak-form linear equation in $\phi$:
\[
R_{\phi_\star}[\mu_t^{N,m}] - \sigma \frac{d}{dt} m^{\bX_t^m} = g_t^{N,m}, \qquad g_t^{N,m} := -\partial_t \mu_t^{N,m} + \frac{\sigma^2}{2} \Delta \mu_t^{N,m},
\]
where $g_t^{N,m}$ depends only on $\mu_t^{N,m}$ and $\sigma$, Here, the differential of the martingale $\frac{d}{dt} m^{\bX_t^m}$, which is unobservable from data and has mean zero, is interpreted as noise. 

% Averaging over the Brownian paths (the martingale drops out) gives
% \[
% \mathbb{E}[R_{\phi_\star}[\mu_t^{N,m}]] = \mathbb{E}[g_t^{N,m}].
% \]
A key structural property is that $R_\phi[u]$ is \emph{linear} in $\phi$, despite being nonlinear in $u$. This linearity, together with the goal of fitting the data $u$ to the weak-form equation, directs us to construct a quadratic loss function of $\phi$ by testing the equation against a family of test functions that depend \emph{linearly} on $\phi$ and depend on the data $u$. Denote such a family by $\{f_\phi[u]\}_\phi$. The loss function will have quadratic form 
\[
\calE(\phi) = \tfrac{1}{2}B_u(\phi,\phi) - \innerp{g, f_\phi[u]}, \quad B_u(\phi,\psi) \;:=\; \innerp{R_\phi[u],\, f_\psi[u]}, 
\]
where $\innerp{\cdot,\cdot}$ denotes the duality pairing between distributions and test functions. In this variational formulation, to make the true potentials $\phi_\star$ a minimizer of the loss, we require the bilinear form $B_u(\phi,\psi)$ to 
be \emph{symmetric} and \emph{non-negative}, so that 
\[
\calE(\phi) - \calE(\phi_\star) = \tfrac{1}{2}B_u(\phi-\phi_\star, \phi-\phi_\star) \geq 0 \qquad \text{for all } \phi. 
\]

These conditions strongly constrain the test family. Following \cite{gao2024self}, consider the self-testing family
\[
f_\phi[u] = V + \Phi * u,
\]
the total potential felt by a particle in the configuration $u$. This choice satisfies all three conditions:
\begin{itemize}
\item \emph{Linearity}: $f_{\phi+\psi}[u] = f_\phi[u] + f_\psi[u]$, so $B_u(\phi,\psi)$ is bilinear in $(\phi,\psi)$.
\item \emph{Symmetry}: for $\phi = (V,\Phi)$ and $\psi = (\tilde V,\tilde\Phi)$, integration by parts gives
\[
B_u(\phi,\psi) = \int_{\R^d} u \,\nabla(V + \Phi * u) \cdot \nabla(\tilde V + \tilde \Phi * u) \, dx = B_u(\psi,\phi).
\]
\item \emph{Non-negativity}:\, $B_u(\phi,\phi) = \int_{\R^d} u \,|\nabla(V + \Phi * u)|^2 \, dx \geq 0$.
\end{itemize}

Intuitively, this is the infinite-dimensional analogue of minimizing $\tfrac{1}{2}x^\top\! Ax - b^\top\! x$ for a symmetric positive-semidefinite system $Ax = b$. The linearity of $f_\phi[u]$ and the symmetry of $B_u$ ensure that the Fr\'echet derivative of the loss at $\phi_\star$ is zero,   
\[
 \lim_{\epsilon \to 0} \frac{\calE(\phi_\star + \epsilon \phi) - \calE(\phi_\star)}{\epsilon} = \innerp{R_{\phi_\star}[u] - g,\, f_\phi[u]} = 0 \qquad \text{for all } \phi, 
\]
that is, $R_{\phi_\star}[u] = g$ tested against the family $\{f_\phi[u]\}_\phi$,  which is the deterministic part of the weak-form PDE~\eqref{eq:dist-PDE} tested against the self-testing family.

Applying this construction to the empirical distributions $\{\mu_{t_\ell}^{N,m}\}$ with Riemann-sum approximation for the time integral $\int_{t_\ell}^{t_{\ell+1}} \frac{1}{2} \innerp{R_{\phi}[\mu_{t}^{N,m}], f_{\phi}[\mu_{t}^{N,m}]}dt$ and averaging over the dataset $\mathcal{D}$, we obtain the \emph{self-test loss function}:
\begin{equation}\label{eq:loss_self_test}
\calE_{\mathcal{D}}(\Phi,V) = \frac{1}{MT}\sum_{m,\ell=1}^{M,L} \Big[\frac{1}{2}\innerp{R_{\phi}[\mu_{t_\ell}^{N,m}], f_{\phi}[\mu_{t_\ell}^{N,m}]}\Delta t -  \int_{t_\ell}^{t_{\ell+1}}\innerp{g_{t}^{N,m}, f_{\phi}[\mu_{t}^{N,m}]} \Big].
\end{equation}
To compute the term $\innerp{g_t^{N,m}, f_\phi[\mu_t^{N,m}]}$ with $g_t^N= -\partial_t \mu_t^N + \frac{\sigma^2}{2} \Delta \mu_t^N$, note first that  
\begin{equation*}
\innerp{-\partial_t \mu_t^N, V+\Phi * \mu_t^N}
= -\frac{d}{dt}\int_{\R^d}\Big[V+\frac{1}{2}\Phi * \mu_t^N\Big]\mu_t^N\,dx, 
\end{equation*}
where we used the fact that  $\frac{d}{dt}\innerp{\mu_t^N,\, \Phi * \mu_t^N}
= \innerp{\partial_t\mu_t^N,\, \Phi * \mu_t^N}
+ \innerp{\mu_t^N,\, \Phi * \partial_t\mu_t^N} = 2 \innerp{\partial_t\mu_t^N,\, \Phi * \mu_t^N}$ since $\Phi$ is symmetric (i.e., $\Phi(x{-}y) = \Phi(y{-}x)$). Also, by integration by parts,
\begin{equation*}
\innerp{\frac{\sigma^2}{2} \Delta \mu_t^N, V+\Phi * \mu_t^N} = \frac{\sigma^2}{2} \int_{\R^d} \mu_t^N [\Delta V + \Delta \Phi * \mu_t^N] dx.
\end{equation*}

Substituting them into \eqref{eq:loss_self_test} and using the Riemann-sum approximation for the time integral $\int_{t_\ell}^{t_{\ell+1}} \innerp{\frac{\sigma^2}{2} \Delta \mu_t^N, V+\Phi * \mu_t^N} dt $, we obtain
\begin{equation*}
\begin{aligned}
\calE_{\mathcal{D}}(\Phi,V) 
= & \frac{1}{MT}\sum_{m,\ell=1}^{M,L} \bigg[ \frac{1}{2} \int_{\R^d} \mu_{t_\ell}^{N,m} |\nabla V + \nabla \Phi * \mu_{t_\ell}^{N,m}|^2 dx \,\Delta t \\
& - \frac{\sigma^2}{2} \int_{\R^d} \mu_{t_\ell}^{N,m} [\Delta V + \Delta \Phi * \mu_{t_\ell}^{N,m}] dx \,\Delta t 
 + \int_{\R^d} \mu_{t_\ell}^{N,m} \big[ V + \frac{1}{2} \Phi * \mu_{t_\ell}^{N,m} \big] dx \Big|_{t_\ell}^{t_{\ell+1}} \bigg].
\end{aligned}
\end{equation*}
The three terms are, respectively, $\tfrac{1}{2}J_{diss}\,\Delta t$, $-\tfrac{\sigma^2}{2}J_{diff}\,\Delta t$, and $\delta E_{f}$, recovering the decomposition~\eqref{eq:loss_trajFree}--\eqref{eq:loss_terms}.

\begin{remark}[Negative minimum of the loss]\label{rem:negative_minimum}
The minimum of the loss function is negative. To see this, consider the simplest setting with $M=1$, $L=1$, and $g = R_{\phi_\star}[u]$ (i.e., without noise). The loss function becomes
\[
  \calE_{\mathcal{D}}(\phi) = \frac{1}{2}\innerp{R_{\phi}[u],\, f_{\phi}[u]} - \innerp{g,\, f_{\phi}[u]} =
  \frac{1}{2}\innerp{R_{\phi-\phi_{\star}}[u],\, f_{\phi-\phi_\star}[u]} - \frac{1}{2}\innerp{R_{\phi_{\star}}[u],\, f_{\phi_\star}[u]},
\]
where the second equality uses the symmetry and bilinearity of $B_u$. The first term is non-negative and equals zero if $\phi = \phi_\star$. Thus the minimum of the loss is $-\frac{1}{2}\innerp{R_{\phi_{\star}}[u],\, f_{\phi_\star}[u]}$, which is strictly negative. In the presence of small perturbations from discretization or the martingale term, the minimum deviates from this value but remains negative. 
\end{remark}

\begin{remark}[Numerical integration in time]\label{rmk:time_integration}
 The Riemann-sum approximation for the time integral in the loss function has a discretization error of order $O(\Delta t^2)$ per interval, and it leads to an $O(\Delta t)$ bias in the parameter estimation (see Theorem~\ref{thm:error_bound}). One can reduce this bias by using a higher-order quadrature rule, such as the trapezoidal rule, which reduces the per-interval error to $O(\Delta t^3)$ and the parameter estimation bias to $O(\Delta t^2)$ (see Corollary~\ref{cor:trapezoid_bound}). 
\end{remark}

\begin{remark} [Mean-field limit of the loss function] 
    Notably, $N$ is finite in the above loss function. When $N\to \infty$, by the propagation of chaos \cite{sznitman1991topics}, the empirical distribution $\mu_t^N$ converges to $u_t$ almost surely, where $u_t$ is the solution to the mean-field equation
    $$
    \partial_t u_t = \nabla \cdot [u_t (\nabla V + \nabla \Phi * u_t)] + \frac{\sigma^2}{2} \Delta u_t,
    $$
   for $t\in [0,T]$. Correspondingly, the three integrals in the loss function converge almost surely to their mean-field limits, and the loss function $\calE_{\mathcal{D}}(\Phi,V)$ with $M=1$ converges almost surely to
    \begin{equation}\label{eq:loss_limit}
    \begin{aligned}
    \calE(\Phi,V) = & \frac{1}{2}\int_0^T \int_{\R^d} |\nabla V(x) + \nabla \Phi * u_t(x)|^2 u_t(x) dx dt - \frac{\sigma^2}{2} \int_0^T \int_{\R^d} [\Delta V(x) + \Delta \Phi * u_t(x)] u_t(x) dx dt \\
    & + \int_{\R^d} [V(x) + \tfrac{1}{2}\Phi * u_T(x)] u_T(x) dx - \int_{\R^d} [V(x) + \tfrac{1}{2}\Phi * u_0(x)] u_0(x) dx.
    \end{aligned}
    \end{equation}
    This is the probabilistic loss function of the mean-field equation and has been studied in \cite{LangLu22} for the recovery of radial interaction potentials.
\end{remark}

\begin{remark}[Energy balance-based loss function.]
Energy balance of the particle system does not directly lead to the self-test loss function. The square mean residual of the energy balance leads to a quartic loss function that has multiple minima and is difficult to optimize. Specifically, let $E_{f}(\bX_t)$ be the free energy defined in \eqref{eq:loss_terms}. The SDE in \eqref{eq:opt_model_R} can be written as gradient flow with noise:
\[ d \bX_t = - N \nabla_{\bX} E_{f}(\bX_t) dt + \sigma d \bW_t. \]
Applying It\^{o}'s formula to $E_{f}(\bX_t)$, we obtain:
\begin{equation}\label{eq:energy_balance}
\begin{aligned}
d E_{f}(\bX_t) & = \nabla_{\bX} E_{f} \cdot d\bX_t + \frac{\sigma^2}{2} \Delta_{\bX} E_{f} dt \\
& = \nabla_{\bX} E_{f} \cdot (- N \nabla_{\bX} E_{f} dt + \sigma d \bW_t) + \frac{\sigma^2}{2} \Delta_{\bX} E_{f} dt \\
& = - N |\nabla_{\bX} E_{f}|^2 dt + \frac{\sigma^2}{2} \Delta_{\bX} E_{f} dt + \sigma \nabla_{\bX} E_{f} \cdot d \bW_t.
\end{aligned}
\end{equation} 
Integrating over $[t_\ell, t_{\ell+1}]$ and noting that $ N |\nabla_{\bX} E_{f}|^2 = \frac{1}{N} \sum_{i=1}^N \left| \nabla V(X_t^i) + \frac{1}{N} \sum_j \nabla \Phi(X_t^i - X_t^j) \right|^2 = J_{diss}(\bX_t)$ and 
$ \Delta_{\bX} E_{f} = \frac{1}{N} \sum_{i=1}^N \left( \Delta V(X_t^i) + \frac{1}{N} \sum_j \Delta \Phi(X_t^i - X_t^j) \right) = J_{diff}(\bX_t)$, we have 
\[ E_{f}(\bX_{t_{\ell+1}}) - E_{f}(\bX_{t_\ell}) + \int_{t_\ell}^{t_{\ell+1}} J_{diss}(\bX_t) dt - \frac{\sigma^2}{2} \int_{t_\ell}^{t_{\ell+1}} J_{diff}(\bX_t) dt = \int_{t_\ell}^{t_{\ell+1}} \sigma \nabla_{\bX} E_{f} \cdot d \bW_t. \]
Since the martingale term has a zero mean, the square mean error of the energy balance is:
\begin{equation*}
\begin{aligned}
\mathcal{L}_{EB}(V,\Phi):=\left| \mathbb{E}\left[  E_{f}(\bX_{   t_{\ell+1}}) - E_{f}(\bX_{t_\ell}) + \int_{t_\ell}^{t_{\ell+1}} J_{diss}(\bX_t) dt - \frac{\sigma^2}{2} \int_{t_\ell}^{t_{\ell+1}} J_{diff}(\bX_t) dt  \right] \right|^2.
\end{aligned}
\end{equation*}
The function $\mathcal{L}_{EB}$ is quartic in the potentials $V$ and $\Phi$ because $J_{diss}$ is quadratic in the potentials. In particular, it has multiple minimizers: in addition to the true potentials, $(V,\Phi) \equiv (0,0)$ is another minimizer of $\mathcal{L}_{EB}(V,\Phi)$ since $\mathcal{L}_{EB}(0,0)= 0$. Thus, it has multiple minima and is non-convex, which may lead to optimization difficulties. In contrast, the self-test loss function is quadratic in the potentials, which leads to better optimization properties. 
\end{remark}

%%%%%%%%%%%%%%%%%
\subsection{Extension to discrete-time models}\label{sec:discretization}
\label{sec:discreteModel}% backward compat
The self-test loss function also applies to discrete-time models of the system that depends linearly on the unknown potentials, such as the Euler-Maruyama discretization of the SDE in \eqref{eq:opt_model_R}. In this case, the loss function is derived from the discrete-time weak-form PDE of the empirical distribution, and it has the same form as \eqref{eq:loss_trajFree}. This extension is important for practical applications where the data are collected at discrete time points and the underlying dynamics may be better described by a discrete-time model. 

Specifically, consider the time series model 
\begin{equation}\label{eq:discrete_model}
X_{t_{\ell+1}}^i = X_{t_\ell}^i - \nabla V(X_{t_\ell}^i) \Delta t - \frac{1}{N} \sum_{j=1}^N \nabla \Phi(X_{t_\ell}^i - X_{t_\ell}^j) \Delta t + \sigma \sqrt{\Delta t} Z_\ell^i, \quad i=1,\ldots,N,
\end{equation}
where $Z_\ell^i$ are i.i.d.~standard normal random variables and $\Delta t<1$. It is from the Euler-Maruyama discretization of the SDE in \eqref{eq:opt_model_R}. 

We first derive the weak-form PDE for the discrete model by Taylor expansion. Note that $ X_{t_{\ell+1}}^i - X_{t_\ell}^i$ is dominated by the diffusion term $\sigma \sqrt{\Delta t} Z_\ell^i$, particularly when $\Delta t$ is small. For $f\in C^2_b(\R^d)$, Taylor expansion gives, 
\begin{equation*}
\begin{aligned}
f(X_{t_{\ell+1}}^i) & = f(X_{t_\ell}^i) + \nabla f(X_{t_\ell}^i) \cdot (X_{t_{\ell+1}}^i - X_{t_\ell}^i) + \frac{1}{2} (X_{t_{\ell+1}}^i - X_{t_\ell}^i)^\top \nabla^2 f(X_{t_\ell}^i) (X_{t_{\ell+1}}^i - X_{t_\ell}^i) + o(\Delta t)\\
& = f(X_{t_\ell}^i) + \nabla f(X_{t_\ell}^i) \cdot \bigg[- \nabla V(X_{t_\ell}^i) \Delta t - \frac{1}{N} \sum_{j=1}^N \nabla \Phi(X_{t_\ell}^i - X_{t_\ell}^j) \Delta t + \sigma \sqrt{\Delta t} Z_\ell^i\bigg] \\ 
& \quad + \frac{\sigma^2 \Delta t}{2} Z_\ell^{i\top} \nabla^2 f(X_{t_\ell}^i) Z_\ell^i
  + o(\Delta t),  
\end{aligned}
\end{equation*}
where the $o(\Delta t)$ term is of order $(\Delta t)^{3/2}$, and we kept only the $\frac{\sigma^2 \Delta t}{2} Z_\ell^{i\top} \nabla^2 f(X_{t_\ell}^i) Z_\ell^i$ term in the second-order expansion since it is of order $\Delta t$ and dominates the other second-order terms that are of order $(\Delta t)^2$. Summing over $i$ and dividing by $N$, we have
\begin{equation*}
\begin{aligned}
  \frac{1}{N} \sum_{i=1}^N f(X_{t_{\ell+1}}^i) & = \frac{1}{N} \sum_{i=1}^N f(X_{t_\ell}^i) -  \Delta t \frac{1}{N} \sum_{i=1}^N \nabla f(X_{t_\ell}^i) \cdot \big[ \nabla V(X_{t_\ell}^i) + \frac{1}{N} \sum_{j=1}^N \nabla \Phi(X_{t_\ell}^i - X_{t_\ell}^j) \big] \\
     & \quad + \frac{\sigma^2 \Delta t}{2}  \frac{1}{N} \sum_{i=1}^N Z_\ell^{i\top} \nabla^2 f(X_{t_\ell}^i) Z_\ell^i +  \sigma \sqrt{\Delta t} \frac{1}{N} \sum_{i=1}^N \nabla f(X_{t_\ell}^i) \cdot Z_\ell^i +  o(\Delta t). 
\end{aligned}
\end{equation*}
Approximating the second-order term by its expectation since the $Z_\ell^i$ are i.i.d.~standard normal random variables, we have 
\[
\frac{1}{N} \sum_{i=1}^N Z_\ell^{i\top} \nabla^2 f(X_{t_\ell}^i) Z_\ell^i \approx   \frac{1}{N} \sum_{i=1}^N \text{tr}(\nabla^2 f(X_{t_\ell}^i)) = \frac{1}{N} \sum_{i=1}^N \Delta f(X_{t_\ell}^i).
\]
Hence, the above equation can be written as 
\begin{equation*}
\begin{aligned}
  \frac{1}{N} \sum_{i=1}^N f(X_{t_{\ell+1}}^i) & = \frac{1}{N} \sum_{i=1}^N f(X_{t_\ell}^i) - \Delta t \frac{1}{N} \sum_{i=1}^N \nabla f(X_{t_\ell}^i) \cdot \big[ \nabla V(X_{t_\ell}^i)  + \frac{1}{N} \sum_{j=1}^N \nabla \Phi(X_{t_\ell}^i - X_{t_\ell}^j)\big] \\
     & \quad + \frac{\sigma^2 \Delta t}{2} \frac{1}{N} \sum_{i=1}^N \Delta f(X_{t_\ell}^i) + \sigma \sqrt{\Delta t}  \frac{1}{N} \sum_{i=1}^N \nabla f(X_{t_\ell}^i) \cdot Z_\ell^i  +  O(\Delta t)
\end{aligned}
\end{equation*}
for any $f \in C^2(\mathbb{R}^d)$, where the $ O(\Delta t)$ term is dominated by the bias $\frac{\sigma^2 \Delta t}{2} \big[ \frac{1}{N} \sum_{i=1}^N Z_\ell^{i\top} \nabla^2 f(X_{t_\ell}^i) Z_\ell^i - \frac{1}{N} \sum_{i=1}^N \Delta f(X_{t_\ell}^i) \big]$.
 Thus, we can write it as a discrete-time weak-form PDE of the empirical distribution $\mu_t^N$ of the discrete model: 
\begin{equation}\label{eq:discrete_PDE}
\begin{aligned}
    \innerp{\mu_{t_{\ell+1}}^N, f} - \innerp{\mu_{t_\ell}^N, f} 
& =  \innerp{R_\phi[\mu_{t_\ell}^N], f} \Delta t +  \innerp{\frac{\sigma^2}{2} \Delta \mu_{t_\ell}^N, f} \Delta t + \sigma  \Delta  m^{\bX_t}(f) + O(\Delta t),  
\end{aligned}
\end{equation}
where $\Delta  m^{\bX_t}(f) = \frac{1}{N} \sum_{i=1}^N \nabla f(X_{t_\ell}^i) \cdot Z_\ell^i \sqrt{\Delta t}$.

The self-test loss function can be derived from the above discrete-time weak-form PDE in the same way as in Section~\ref{sec:derivation_loss_fn}, and it has the same form as \eqref{eq:loss_trajFree}. 

The main difference from the continuous-time SDE model case is that one uses the Taylor expansion (instead of the It\^o formula) to derive the weak-form PDE, and the error from the Taylor expansion leads to an $O(\Delta t)$ bias in the PDE. Since the model itself is discrete, one cannot reduce this bias by using a higher-order quadrature rule for the time integral in the loss function, as one would do for the continuous-time model. Also, this bias will prevent the estimator from converging to the true parameters as $M\to \infty$, unless $\Delta t$ is sufficiently small, which is observed in the \emph{zero-gap regime} in numerical experiments in Figure~\ref{fig:discrete_bias} in Section~\ref{sec:model_a_validation}.

%%%%%%%%%%%%%%%%========= 
 \section{Algorithms: parametric and nonparametric regression}
\label{sec:computation}
The self-test loss function provides a principled way to learn the potentials from the data. It can be optimized using standard optimization algorithms, and its quadratic structure allows for efficient optimization and robust convergence properties. In particular, when the potentials are linearly parameterized, the loss function reduces to a least squares regression problem, which can be solved efficiently with closed-form solutions or iterative solvers. For high-dimensional or complex potentials where nonparametric flexibility is needed, we can use neural network regression to minimize the loss function.

\subsection{Least squares regression with basis functions}
\label{sec:basis_func}
When the potentials $\Phi$ and $V$ are linearly parametrized using basis functions (including kernel ridge regression), the learning problem reduces to solving a linear system. This is a direct consequence of the quadratic structure of the self-test loss function derived in Section \ref{sec:loss_fn}. Specifically, let $\{\psi_k\}_{k=1}^{K_V}$ and $\{\phi_l^{sym}\}_{l=1}^{K_\Phi}$ be linearly independent basis functions in $C^2(\R^d)$, which can be either a prescribed dictionary of functions or radial basis functions, we expand the potentials with coefficients $\theta = (\alpha, \beta)$:
\begin{equation}\label{eq:basis_expansion}
    V_{\alpha}(x) = \sum_{k=1}^{K_V} \alpha_k \psi_k(x), \quad \Phi_{\beta}(x) = \sum_{l=1}^{K_\Phi} \beta_l \phi_l^{sym}(x). 
\end{equation}
Then, the loss function in \eqref{eq:loss_trajFree} can be written as a function of the coefficients $\theta$:
\begin{equation}
\label{eq:loss_basis}
\mathcal{L}(\theta) = \frac{1}{MT}\sum_{\ell=0}^{L-1}\sum_{m=1}^M \left( \frac{1}{2} Q_{\ell,m}(\theta) \Delta t - \frac{\sigma^2}{2} L_{\ell,m}(\theta) \Delta t +  E_{\ell+1,m}(\theta) - E_{\ell,m}(\theta) \right),
\end{equation}
where the terms are obtained by substituting the basis expansions into the terms defined in \eqref{eq:loss_terms}: the dissipation term $Q_{\ell,m}(\theta) = N^{-1}\sum_i |\nabla V_\alpha(X_{t_\ell}^{i,m}) + \frac{1}{N}\sum_{j \neq i} \nabla \Phi_\beta (X_{t_\ell}^{i,m}-X_{t_\ell}^{j,m})|^2$ is a \emph{quadratic} in $\theta$; the diffusion correction 
$L_{\ell,m}(\theta) = \frac{1}{N}\sum_{i=1}^N \Delta V_\alpha(X_{t_\ell}^{i,m}) + \frac{1}{N^2}\sum_{i\neq j} \Delta \Phi_\beta(X_{t_\ell}^{i,m}-X_{t_\ell}^{j,m})$ and the potential energy $E_{\ell,m}(\theta)= \frac{1}{N}\sum_{i=1}^N V_\alpha(X_{t_\ell}^{i,m}) + \frac{1}{2N^2}\sum_{i\neq j}  \Phi_\beta(X_{t_\ell}^{i,m}-X_{t_\ell}^{j,m}) $ are \emph{linear} in $\theta$.

Consequently, the loss function is a quadratic polynomial in $\theta$ of the form:
\begin{equation}\label{eq:loss_theta}
    \mathcal{L}(\theta) = \frac{1}{2} \theta^\top \mathbf{A}_{M,\Delta t}  \theta - \mathbf{b}_{M,\Delta t} ^\top \theta, 
\end{equation}
where the normal matrix $\mathbf{A}_{M,\Delta t} $ is assembled from the cross-terms of the basis function gradients in $Q_\ell$, and the normal vector $\mathbf{b}_{M,\Delta t} $ aggregates the contributions from the linear energy difference and diffusion terms. They are computed by  
\begin{equation}\label{eq:normal_matrix_vector}
    \mathbf{A}_{M,\Delta t}  = \frac{1}{MLN} \sum_{m=1}^M\sum_{\ell=0}^{L-1}  \sum_i (\mathbf{F}_{i,t_\ell}^{m})^\top \mathbf{F}_{i,t_\ell}^{m}, \quad \mathbf{b}_{M,\Delta t}  = 
    \frac{1}{MT} \sum_{m=1}^M\sum_{\ell=0}^{L-1}\big[ \frac{\sigma^2}{2} \;\boldsymbol{\delta}_{t_\ell}^{m} \Delta t- (\mathbf{h}_{t_{\ell+1}}^{m} - \mathbf{h}_{t_\ell}^{m}) \big].  
\end{equation}
Here, for each sample $\bX_{t_\ell}^m$, $
\mathbf{F}_{i,t_\ell}^m := \mathbf{F}_i(\bX_{t_\ell}^m)$, $\boldsymbol{\delta}_{t_\ell}^m := \boldsymbol{\delta}(\bX_{t_\ell}^m)$, and $\mathbf{h}_{t_\ell}^m := \mathbf{h}(\bX_{t_\ell}^m)$, 
where for any particle configuration
$\bX=(X^1,\ldots,X^N)\in (\R^d)^N$, we introduce vector-valued regression functions for the forces, diffusion correction and potential energy:
\[
\mathbf{F}_i(\bX)\in\R^{d\times K},\quad i=1,\ldots,N; \quad
\boldsymbol{\delta}(\bX)\in\R^K; \quad
\mathbf{h}(\bX)\in\R^K 
\]
with components
\begin{equation}\label{eq:regressionVectors}
\begin{aligned}
\mathbf{F}_i(\bX)_{:,k}
& =
\begin{cases}
\nabla \psi_k(X^i), & 1 \le k \le K_V, \\[3pt]
 \frac{1}{N}\sum_{j\neq i}\nabla \phi^{sym}_{k-K_V}(X^i-X^j), & K_V < k \le K,
\end{cases}  \\ % [6pt]
\boldsymbol{\delta}(\bX)_k
& =
\begin{cases}
\frac{1}{N}\sum_{i=1}^N \Delta \psi_k(X^i),  & 1 \le k \le K_V, 
\\[6pt]
 \frac{1}{N^2}\sum_{i\neq j}\Delta \phi^{sym}_{k-K_V}(X^i-X^j), 
& K_V < k \le K,
\end{cases}  \\[6pt]
\hspace{-3mm}
\mathbf{h}(\bX)_k
& =
\begin{cases}
 \frac{1}{N}\sum_{i=1}^N \psi_k(X^i), & 1 \le k \le K_V, \\[6pt]
 \frac{1}{2N^2}\sum_{i\neq j}\phi^{sym}_{k-K_V}(X^i-X^j), & K_V < k \le K.
\end{cases}
\end{aligned}
\end{equation}

Minimizing $\mathcal{L}(\theta)$ is therefore equivalent to solving the linear system (normal equations) with proper regularization if necessary:
\begin{equation}\label{eq:normalEq}
    \mathbf{A}_{M,\Delta t}  \widehat{\theta}_{M,\Delta t} =  \mathbf{b}_{M,\Delta t} .
\end{equation}
Thus, the estimator $\widehat{\theta}_{M,\Delta t}$ can be obtained by a single linear solve, which is computationally efficient and numerically stable for moderate $K$.

\paragraph{Regularization.} When the number of basis functions $K$ is large (or when they are not carefully chosen) and the data are noisy, the normal equations can be ill-conditioned, leading to unstable solutions. In this case, regularization is necessary to ensure numerical stability and reliable parameter estimation. 

In particular, the estimation of the interaction potential $\Phi$ is more ill-posed than that of the confining potential $V$ because the data provide indirect information about $\Phi$ through the average of $N$ pairwise interactions, while $V$ can be inferred more directly from the individual particle positions. This is reflected in the coercivity condition in Proposition~\ref{prop:joint_coercivity}: the interaction block is weighted by $c_N=(N-1)/N^2=O(1/N)$, whereas the confining-potential block has $O(1)$ coercivity. When the sample size is finite, this leads to the block corresponding to $\Phi$ in the normal matrix $\mathbf{A}_{M,\Delta t}$ having smaller eigenvalues and a larger condition number than the block corresponding to $V$. 

% To see why, consider the block structure of $\mathbf{A}_{M,\Delta t}$ corresponding to the partition $\theta = (\alpha, \beta)$:
% \[
% \mathbf{A}_{M,\Delta t} = \begin{pmatrix} \mathbf{A}_{VV} & \mathbf{A}_{V\Phi} \\ \mathbf{A}_{\Phi V} & \mathbf{A}_{\Phi\Phi} \end{pmatrix}.
% \]
% The diagonal block $\mathbf{A}_{VV}$ is assembled from $\nabla V$ terms evaluated at individual particle positions. Since each of the $N$ particles contributes independently, $\mathbf{A}_{VV}$ accumulates $MN$ effectively independent samples per time step, making it well-conditioned. In contrast, the block $\mathbf{A}_{\Phi\Phi}$ is assembled from $\nabla \Phi$ terms evaluated at pairwise differences $X_i - X_j$. While there are $N(N-1)/2$ distinct pairs, these pairs share particles and are thus correlated. Moreover, the interaction force enters the loss through the ensemble average $\frac{1}{N}\sum_{j} \nabla \Phi(X_i - X_j)$, which introduces a factor of $1/N^2$ in the dissipation term. Consequently, $\mathbf{A}_{\Phi\Phi}$ has smaller eigenvalues and a larger condition number than $\mathbf{A}_{VV}$.

To stabilize the solution, we apply Tikhonov (ridge) regularization:
\begin{equation*}% \label{eq:tikhonov}
(\mathbf{A}_{M,\Delta t} + \lambda \mathbf{I}) \hat{\theta} = \mathbf{b}_{M,\Delta t},
\end{equation*}
and we select the hyperparameter  $\lambda > 0$ via the Hansen L-curve method \cite{hansen1992analysis}. Specifically, we compute the SVD of $\mathbf{A}_{M,\Delta t}$, evaluate the L-curve (log residual norm vs.\ log solution norm) analytically over a grid of candidate $\lambda$ values, and choose the $\lambda$ that maximizes the curvature $\kappa(\lambda)$ of the curve $(\log \|\mathbf{A}_{M,\Delta t} \hat{\theta}_\lambda - \mathbf{b}_{M,\Delta t}\|, \log \|\hat{\theta}_\lambda\|)$. When the curve is flat (curvature $|\kappa| < 0.01$, indicating an absence of a L-shaped corner), we fall back to a minimal regularization $\lambda = 10^{-6}$.

\smallskip
We summarize the procedure in Algorithm \ref{alg:basis_regression}. The computational cost of assembling $\mathbf{A}_{M,\Delta t}$ and $\mathbf{b}_{M,\Delta t}$ is $O(M L N^2 K^2)$ due to the pairwise interactions. The linear solve step has a cost of $O(K^3)$, which is negligible compared to the assembly step for moderate $K$.  

\begin{algorithm}
\caption{Least squares regression with basis functions}
\label{alg:basis_regression}
\begin{algorithmic}[1]
\STATE \textbf{Input:} Unlabeled ensembles $\mathcal{D} = \{\bX_{t_\ell}^{1:M}\}_{\ell=0}^{L}$, basis functions $\{\psi_k\}_{k=1}^{K_V}$ and $\{\phi_l^{sym}\}_{l=1}^{K_\Phi}$.

\STATE \textbf{Assemble:} Assemble the matrices $\mathbf{A}_{M,\Delta t}$ and $\mathbf{b}_{M,\Delta t}$ from the basis functions and the unlabeled data in parallel for each snapshot pair $(\ell, m) \in \{0,\dots,L{-}1\}\times\{1,\dots,M\}$.

\STATE \textbf{Solve:} $\hat{\theta} = (\mathbf{A}_{M,\Delta t} + \lambda\,\mathbf{I})^{-1}\,\mathbf{b}_{M,\Delta t}$ with $\lambda$ selected via the Hansen L-curve method.

\STATE \textbf{Output:} Recovered potentials $V_{\hat{\alpha}}$, $\Phi_{\hat{\beta}}$ with $\hat{\theta} = (\hat{\alpha}, \hat{\beta})$.
\end{algorithmic}
\end{algorithm}

% \noindent\textbf{Algorithm 1: Least squares regression with basis functions} \\
% \noindent\textit{Input:} Unlabeled ensembles $\mathcal{D} = \{\bX_{t_\ell}^{1:M}\}_{\ell=0}^{L}$, basis functions $\{\psi_k\}_{k=1}^{K_V}$ and $\{\phi_l^{sym}\}_{l=1}^{K_\Phi}$.
% \begin{enumerate}
%     \item \textbf{Assemble. } Assemble the matrices $\mathbf{A}_{M,\Delta t}$ and $\mathbf{b}_{M,\Delta t}$ from the basis functions and the unlabeled data in parallel for each snapshot pair $(\ell, m) \in \{0,\dots,L{-}1\}\times\{1,\dots,M\}$. 

%     \item \textbf{Solve:} $\hat{\theta} = (\mathbf{A}_{M,\Delta t} + \lambda\,\mathbf{I})^{-1}\,\mathbf{b}_{M,\Delta t}$ with $\lambda$ selected via the Hansen L-curve method. 

%     \item \textbf{Output:} Recovered potentials $V_{\hat{\alpha}}$, $\Phi_{\hat{\beta}}$ with $\hat{\theta} = (\hat{\alpha}, \hat{\beta})$.
% \end{enumerate}

\subsection{Neural network regression}
\label{sec:NN_regression}
When the dimension is high or the potentials are complex and suitable basis functions are unknown, we employ a deep neural network to approximate the potentials. Let $V_{\alpha}(x)$ and $\Psi_{\beta}(x)$ be parameterized by a neural network, where $\Phi_{\beta}(x) = \frac{1}{2}(\Psi_{\beta}(x) + \Psi_{\beta}(-x))$ to enforce evenness ($\Phi(x) = \Phi(-x)$). Our loss function depends on the gradients and the Laplacians of the potential, which we compute using Automatic Differentiation (AD).

We compute the loss function efficiently using batches of $\ell,m$, which allows us to leverage parallelism and reduce memory usage, since we can compute the terms in \eqref{eq:loss_terms} for each unlabeled sample pair $(\bX_{t_\ell}^m, \bX_{t_{\ell+1}}^{m})$. 

The network parameters $\theta = (\alpha, \beta)$ are optimized using stochastic gradient descent with the Adam optimizer. The computational cost per iteration is $O(B N^2)$ due to the pairwise interactions, where $B$ is the batch size. The number of iterations depends on the optimization dynamics and convergence criteria. In particular, since the minimizer of the loss function is negative (see Remark \ref{rem:negative_minimum}), we use \emph{early stopping} based on the flattening of the loss curve and the stabilization of the validation loss to prevent overfitting.

We summarize the procedure in Algorithm \ref{alg:nn_self_test}. 

\begin{algorithm}
\caption{Neural Network Estimation via Trajectory-Free Loss}
\label{alg:nn_self_test}
\begin{algorithmic}[1]
\STATE \textbf{Input:} Unlabeled ensembles $\mathcal{D} = \{\bX_{t_\ell}^{m}\}_{\ell,m=1}^{L,M}$, batch size $B$, and termination criteria (e.g., max epochs, early stopping).

\STATE \textbf{Initialize:} Network parameters $\theta= (\alpha,\beta)$.

\STATE \textbf{Optimization Loop:} Until termination criteria are met, perform the following steps:
\begin{enumerate}[label=(\alph*), noitemsep, topsep=2pt]
    \item Sample a batch of indices $\mathcal{I}_{batch} \subset \{1, \dots, L-1\}\times\{1, \dots, M\}$ of size $B$. 
    \item For each $(\ell,m) \in \mathcal{I}_{batch}$, using the unlabeled snapshot pair $(\bX_{t_\ell}^{m}, \bX_{t_{\ell+1}}^{m})$:
    \begin{enumerate}[label=(\roman*), noitemsep, topsep=2pt]
        \item \textbf{Compute Gradients/Laplacians:} For each particle $i$, compute $\nabla V_\alpha(X_{t_\ell}^{i,m})$, $\Delta V_\alpha(X_{t_\ell}^{i,m})$, $\nabla \Phi_\beta$ and $\Delta \Phi_\beta$ using automatic differentiation.
        \item \textbf{Evaluate Integral Estimates:} Compute $J_{diss}^{(\ell,m)}$, $J_{diff}^{(\ell,m)}$, and $\delta E_{f}^{(\ell,m)}$ in batches.
    \end{enumerate}
    \item Loss computation by assembling the batch estimates:
    \begin{equation*}
    \mathcal{L}(\theta) = \frac{1}{B} \sum_{\ell,m \in \mathcal{I}_{batch}}  \left[\frac{1}{2} J_{diss}^{(\ell,m)} \Delta t - \frac{\sigma^2}{2} J_{diff}^{(\ell,m)} \Delta t + \delta E_{f}^{(\ell,m)}\right],
    \end{equation*}
    \item Update: $\theta \leftarrow \theta - \gamma \nabla_\theta \mathcal{L}$ (e.g., Adam optimizer).
\end{enumerate}

\STATE \textbf{Output:} Learned potentials $V_{\alpha}, \Phi_{\beta}$.
\end{algorithmic}
\end{algorithm}

A visual overview of the two algorithms is shown in Figure~\ref{fig:algorithm_combined_flowchart}.

\begin{figure}[ht!]
\centering
\includegraphics[width=0.98\textwidth]{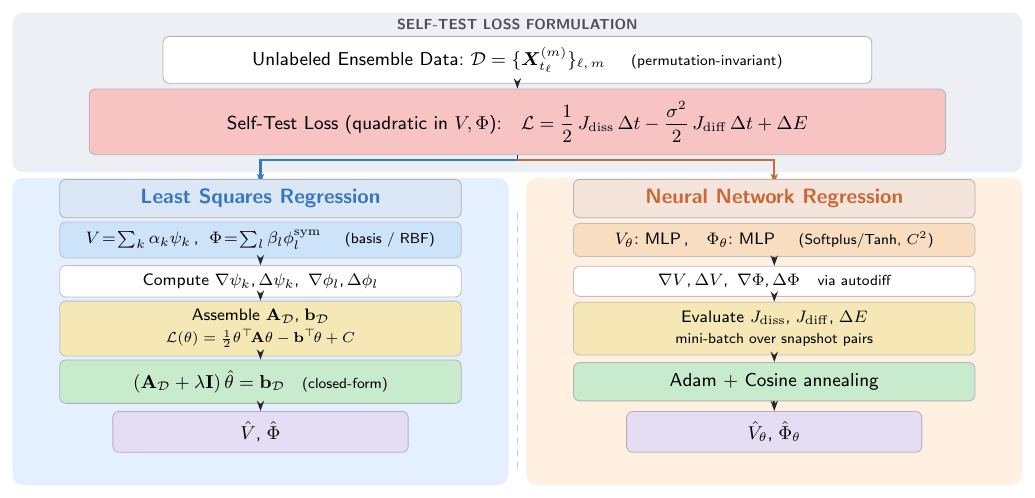}
\caption{Workflow of both estimation algorithms using the self-test loss function.  \textbf{Left:} Least squares regression expands $V$ and $\Phi$ in prescribed basis functions (or RBFs), and solves the minimizer via least squares with regularization. \textbf{Right:} Neural network regression parameterizes $V_\theta$ and $\Phi_\theta$ by neural networks, computes derivatives via automatic differentiation, and minimizes the loss using stochastic gradient descent.}
\label{fig:algorithm_combined_flowchart}
\end{figure}

%%%%%%%%%%%%%
\subsection{Baseline methods and computational complexity}\label{sec:baselines}

To evaluate the performance of the proposed self-test approach, we compare it against several complementary estimation strategies that represent different combinations of assumptions and methodologies. These baselines serve to contextualize the results and highlight the advantages of the self-test approach in various settings.

\textbf{1. MLE with labeled data: ideal upper benchmark.}
This method assumes that the true particle labels are known across time, so trajectories are fully observed. It fits the drift by minimizing the negative log-likelihood using the Euler--Maruyama discretization,
% \[
% v_i^\ell := \frac{X_{t_{\ell+1}}^i - X_{t_\ell}^i}{\Delta t}
% \approx -\nabla V(X_{t_\ell}^i) - \frac{1}{N}\sum_{j\neq i}\nabla \Phi(X_{t_\ell}^i-X_{t_\ell}^j) + \text{noise}.
% \]
$$
\mathcal{L}_{MLE}(\Phi, V) = \sum_{\ell, i} \left\| X_{t_{\ell+1}}^{i} - X_{t_\ell}^i + \big( \frac{1}{N}\sum_{j \neq i} \nabla \Phi(X_{t_\ell}^i - X_{t_\ell}^j) + \nabla V(X_{t_\ell}^i) \big)\Delta t  \right\|^2.
$$
This baseline represents the \emph{best-case trajectory-based estimator}: it uses information that is unavailable in the unlabeled setting, and therefore provides an ideal upper bar for the self-test approach. It is asymptotically optimal when the model is correctly specified and $\Delta t \to 0$. Its limitation is that it cannot be applied to unlabeled data, and its performance degrades rapidly as $\Delta t$ increases due to the bias in the finite-difference velocity estimation. 

\textbf{2. Sinkhorn MLE: practical unlabeled baseline.} 
For unlabeled data, a natural practical approach is to first reconstruct trajectories via label assignment and then apply the MLE regression. To do so, between two consecutive snapshots, we compute an entropically regularized optimal transport coupling with cost
\[
C_{ij} = |X_{t_\ell}^i - X_{t_{\ell+1}}^j|^2,
\]
using the Sinkhorn algorithm \cite{cuturi2013sinkhorn, feydy2019interpolating}, and convert the coupling into matched pairs and hence the particle labels. The resulting pseudo-trajectories are then used in the MLE regression above to obtain potential estimates. This baseline is a practical competitor for \emph{unlabeled} data: it combines a widely used matching procedure with a standard estimator. Its limitation is that matching errors grow quickly with $\Delta t$ and diffusion $\sigma$, leading to biased potential estimates. Importantly, the computational cost of the Sinkhorn algorithm is significant compared to the self-test regression, since it requires solving an optimization problem of label assignment for each snapshot pair; see Table \ref{tab:method_comparison}. 
% $O(N^2ML)$, where $N^2$ arises from the pairwise cost matrix, $M$ is the number of samples, and $L$ is the number of snapshot pairs. This cost is significant compared to the self-test regression, which does not require matching and has a cost of $O(N^2ML)$ for the pairwise interactions in the loss function, but without the additional overhead of the Sinkhorn iterations. 

\smallskip  
To evaluate the performance of the self-test approach relative to these baselines, we consider a linear parametric setting where the potentials are expanded in a known set of basis functions, as described in Section~\ref{sec:basis_func}. In this setting, we can apply the three estimation strategies with the same basis function representation, leading to labeled MLE, self-test regression, and Sinkhorn MLE estimators. This allows for a systematic comparison of the methods under controlled conditions, where the true functional form is known and the only difference is the estimation strategy and the availability of labels. 

In addition, we apply the self-test NN estimator, which does not rely on basis functions and can capture more complex potentials, providing a complementary perspective on the performance of the self-test approach in a nonparametric setting. We do not combine Sinkhorn with NN regression, since inaccurate label recovery at large $\Delta t$ introduces biased potential estimates that a more expressive regressor cannot correct. 

\paragraph{Computational complexity.}\label{sec:method_summary}
Table~\ref{tab:method_comparison} shows the computational complexity of these methods.

\begin{table}[ht]
\centering
\caption{Computational complexity of all estimation methods. Here, $M$ is the number of trajectories, $L$ is the number of time snapshots, $N$ is the number of particles, $d$ is the dimension, $K = K_V {+} K_\Phi$ is the basis size, $P_V, P_\Phi$ are the network parameters, $E$ is the number of epochs, and $T_S$ is the number of Sinkhorn iterations.}
\label{tab:method_comparison}
\small
\begin{tabular}{l c c l}
\toprule
Method &   {Labeled} &  \#Params & Complexity \\
\midrule
Labeled MLE         & \checkmark  & $K$ & \multirow{2}{*}{$O\bigl(MLd\,NK(N{+}K)+K^3\bigr)$} \\
Self-Test LSE  & ---         & $K$ & \\
\cmidrule(lr){1-4}
Sinkhorn MLE    & ---         & $K$ & $O\bigl(MLd[ NK(N{+}K)+ T_Sd N^2] {+}K^3\bigr)$ \\
\midrule
Self-Test NN       & ---         & $P$  & $O\bigl(EMLd(NP_V + N^2P_\Phi)\bigr)$ \\
\bottomrule
\end{tabular}
\end{table}

The labeled MLE is the ideal benchmark with the lowest complexity, since it does not require matching. Its complexity is dominated by the evaluation of pairwise interactions for each particle and each snapshot, leading to $O(N^2)$ per snapshot, and the linear solve step for the basis regression, leading to $O(K^3)$, where $K$ is the number of basis functions. The self-test LSE has the same asymptotic complexity as the labeled MLE, since it also relies on pairwise interactions and a linear solve step, but it does not require labels. 

The Sinkhorn MLE method adds the cost of the Sinkhorn algorithm for label matching, which is $O(T_S d N^2)$ per snapshot, where $T_S$ is the number of Sinkhorn iterations. This significantly increases the computational cost compared to the self-test regression, especially for large $N$ and $ML$. 

The self-test NN method minimizes the loss function via iterative optimization using stochastic gradient descent, and the cost per iteration depends on the network architecture and the number of epochs required for convergence. The cost is dominated by the forward and backward passes through the network for each particle and each snapshot, leading to $O(EMLd(NP_V + N^2P_\Phi))$, where $E$ is the number of epochs, $P_V$ and $P_\Phi$ are the number of parameters in the networks for $V$ and $\Phi$, respectively.

Thus, the computational complexity of the methods increases from the labeled MLE and the self-test LSE (same and lowest), to the Sinkhorn MLE (higher due to matching), and finally to the self-test NN (highest due to iterative optimization).

%%%%%%%%%%%%%%%%=========
\section{Error bounds for the parametric estimator}\label{sec:theory}
This section establishes the convergence and error bounds of the self-test estimator in the linear parametric setting in Section \ref{sec:basis_func}. We show that the coercivity of the dissipation term guarantees the well-posedness of the inverse problem in the large sample limit. The finite-sample error bounds are then derived using matrix perturbation theory, explicitly connecting the sample size $M$, observation time gap $\Delta t$, and particle number $N$. The nonparametric neural network estimator is more challenging to analyze due to the non-convexity of the loss landscape and the implicit regularization effects of the optimization, which we leave for future work.

%%%%%%% 
% \subsection{Parametric setup and main results}
Recall that in the linear parametric setting, the potentials are expanded in a finite set of basis functions (as in \eqref{eq:basis_expansion}): 
\[
    V_\alpha(x) = \alpha^\top \boldsymbol \psi(x),
    \qquad
    \Phi_\beta(x) = \beta^\top \boldsymbol \phi(x),
    \qquad
    \theta=(\alpha,\beta)\in\R^K,
    \qquad
    K=K_V+K_\Phi, 
\]
where the basis functions are collected into the vectors
\[
    \boldsymbol \psi(x) = \bigl(\psi_1(x),\dots,\psi_{K_V}(x)\bigr)^\top,
    \qquad
    \boldsymbol \phi(x) = \bigl(\phi^{sym}_1(x),\dots,\phi^{sym}_{K_\Phi}(x)\bigr)^\top.
\]
The estimator $\widehat\theta_{M,\Delta t}$ solves the normal equation \eqref{eq:normalEq}, i.e., 
$$\widehat\theta_{M,\Delta t} = \mathbf A_{M,\Delta t}^\dag \mathbf b_{M,\Delta t}, $$
where the empirical normal matrix $\mathbf{A}_{M,\Delta t}$ and normal vector $\mathbf{b}_{M,\Delta t}$ are given by 
\begin{equation*}
    \mathbf A_{M,\Delta t}
    := \frac1{ML}\sum_{m=1}^M\sum_{\ell=0}^{L-1}
    \frac1N\sum_{i=1}^N (\mathbf F_{i,t_\ell}^m)^\top \mathbf F_{i,t_\ell}^m, \quad    \mathbf b_{M,\Delta t}
    := \frac1{ML}\sum_{m=1}^M\sum_{\ell=0}^{L-1} \frac{\sigma^2}{2}\boldsymbol\delta_{t_\ell}^m
    - \frac1{MT}\sum_{m=1}^M \bigl(\boldsymbol h_T^m-\boldsymbol h_0^m\bigr).
\end{equation*}
Here, $\mathbf F_{i,t_\ell}^m$, $\boldsymbol\delta_{t_\ell}^m$, and $\boldsymbol h_t^m$ are the regression vectors defined in \eqref{eq:regressionVectors}, which depend on the gradients, Laplacians, and values of the basis functions, evaluated at the sample $\bX_{t_\ell}^m$.

Our goal is to prove that the estimator $\widehat\theta_{M,\Delta t}$ converges to the true parameter $\theta_*= (\alpha_*,\beta_*)$ as $M\to\infty$ and $\Delta t\to 0$.

We make the following assumptions on the data-generating process and the basis functions. 
\begin{assumption}% [Asumptions on the process and the basis functions]
    \label{ass:parametric_theory}
Conditions on the distribution of $(\bX_t)_{t\in[0,T]}$ and the basis functions: 
\begin{itemize}
    \item \textbf{Basis moment bounds.} The basis functions $\psi_k,\phi_\ell^{sym}\in C_b^4(\R^d)$ satisfy 
    \[
     c_{max,p} = \max_{1\le k\le K_V}\|\psi_k\|_{C_b^p(\R^d)}+\max_{1\le \ell\le K_\Phi}\|\phi_\ell^{sym}\|_{C_b^p(\R^d)}<\infty, 2\leq p\leq 4, 
    \]
     each $\phi_\ell^{sym}$ is symmetric, and the time-averaged Gram matrices 
    \begin{equation}\label{eq:GV_GPhi_skeleton}
        \mathbf G_V := \frac1T\int_0^T \mathbb E\big[\nabla\boldsymbol\psi(X_t^1)^\top\nabla\boldsymbol\psi(X_t^1)\big]dt,
        \quad
        \mathbf G_\Phi := \frac1T\int_0^T \mathbb E\big[\nabla\boldsymbol\phi(X_t^{1,2})^\top\nabla\boldsymbol\phi(X_t^{1,2})\big]dt 
    \end{equation} 
    are positive definite, where $\nabla \boldsymbol \psi=(\nabla\psi_1,\dots,\nabla\psi_{K_V})$ and $\nabla\boldsymbol\phi=(\nabla\phi_1^{sym},\dots,\nabla\phi_{K_\Phi}^{sym})$. Also, the true potentials $V_{\alpha_*}$ and $\Phi_{\beta_*}$ are in the span of the basis functions. 
    
    \item \textbf{Exchangeability and conditional independence.} For each $t\in[0,T]$, the law of $\bX_t=(X_t^1,\dots,X_t^N)$ is exchangeable.  For a.e.~$t\in[0,T]$, there exists a sub-$\sigma$-algebra $\mathcal F_t$ such that $X_t^1-X_t^2$ and $X_t^1-X_t^3$ are conditionally independent given $\mathcal F_t$. 

    \item \textbf{Weak cross-correlation.} There exists $c_0\in[0,1)$ s.t.~for every $\theta=(\alpha,\beta)\in\R^{K_V}\times\R^{K_\Phi}$,
    \begin{equation}\label{eq:compatibility_skeleton}
        \bigg|\int_0^T \mathbb E\big[\nabla V_\alpha(X_t^1)\cdot \nabla\Phi_\beta(X_t^1-X_t^2)\big]dt\bigg|
        \le
        \frac{c_0N}{N-1}  \int_0^T   \mathbb E\big[|\nabla V_\alpha(X_t^1)|^2 +c_N |\nabla \Phi_\beta(X_t^1-X_t^2)|^2\big]  dt, 
    \end{equation}
    where $c_N = \frac{N-1}{N^2}$. 
\end{itemize}
\end{assumption}

The first item is a regularity and moment-bounds condition on the basis functions and the data, which are natural assumptions for using It\^o calculus and the Riemann-sum approximation. In particular, when the basis functions are in $C_b^2(\R^d)$, the processes  $\mathbf  h_t= \mathbf h(\bX_t)$, $\boldsymbol\delta_t = \boldsymbol\delta(\bX_t)$, and $\mathbf F_{i,t} = \mathbf F_i(\bX_t)$ in \eqref{eq:regressionVectors} satisfy
    \[
        \mathbb E\big[\sup_{t\in[0,T]} |\boldsymbol h_t|^2\big] < K c_{max,2}^2,
        \quad
        \mathbb E\big[\sup_{t\in[0,T]} |\boldsymbol\delta_t|^2\big] < K c_{max,2}^2,
        \quad
        \mathbb E\big[\sup_{t\in[0,T]} \frac1N\sum_{i=1}^N \|\mathbf F_{i,t}\|_F^2\big] < K^2 c_{max,2}^2. 
    \]
 The $C^4_b$ condition is used to control the time-discretization error in Lemma~\ref{lem:discretization_bound}, and it can be replaced by a weaker condition that the maps $ t\mapsto \mathbb E\left[\frac1N\sum_{i=1}^N \mathbf F_{i,t}^\top \mathbf F_{i,t}\right]$ and $t\mapsto \mathbb E[\boldsymbol\delta_t]$ are Lipschitz continuous on $[0,T]$.
    
The second item on exchangeability and conditional independence is used to establish coercivity, which is essential for well-posedness in estimating the interaction potential. The third item, a weak cross-correlation between the external and interaction terms, ensures the coercivity of the dissipation term, which guarantees that the joint estimation of $V$ and $\Phi$ is well-posed. The weak cross-correlation condition is automatically satisfied with $c_0=0$ if the basis functions for $V$ and $\Phi$ are chosen to be orthogonal in the sense that $\int \nabla\psi_k(x)\cdot\nabla\phi_\ell^{sym}(x-y) d\mu_t(x,y)=0$ for every $k,\ell$ and a.e.~$t\in[0,T]$, where $\mu_t$ is the joint law of $(X_t^1,X_t^2)$. In general, such orthogonality may not hold, but the weak cross-correlation condition allows for some degree of correlation while still ensuring identifiability. Also, the factors $\frac{N}{N-1}$ and $c_N$ in \eqref{eq:compatibility_skeleton} are a technical artifact that arises from the symmetrization of the interaction potential, and their product converges to 0 as $N\to \infty$, agreeing with the ill-posedness in estimating $\Phi$ in the mean-field limit in \cite{LangLu22,LangLu21id,gao2024self}.

\begin{theorem}[Error bound]\label{thm:error_bound}
Under Assumption~{\rm\ref{ass:parametric_theory}}, for every $\eta\in(0,1)$, we have 
\begin{equation}\label{eq:parametric_error}
    |\widehat  \theta_{M,\Delta t}-\theta_*|
    \le
    C\left(\Delta t + \frac{1}{\sqrt{M\eta}}\right)
\end{equation}
with probability at least $1-\eta$, provided $M\ge M_0/\eta$ and $\Delta t\le \Delta_0$ for some constants $C, M_0$ and $\Delta_0$ independent of $M$ and $\Delta t$.
\end{theorem}

The constants in Theorem~\ref{thm:error_bound} can be chosen as $\Delta_0=\frac{\mu_*}{4L_A}$, $M_0 = \frac{32 C_A}{\mu_*^2}$ and $C = \max\{C_1, C_2\}$ with $C_1=\frac{4}{3\mu_*}\big(\frac{\sigma^2L_b}{2}+L_A|\theta_*|\big)$ and $C_2=\frac{4\sqrt{2}}{\mu_*}\max\left\{\sqrt{C_b},\; \sqrt{C_A} (|\theta_*|+ C_1\Delta_0)\right\}$.  
Here, $\mu_*=\lambda_{\min}(\mathbf A_*)$ with $\mathbf A_* $ is the continuum expected normal matrix in \eqref{eq:population_Ab}, and the constants $L_A, L_b, C_A, C_b$ arise in Lemma~\ref{lem:discretization_bound} and Lemma~\ref{lem:concentration} for controlling the discretization error and concentration and are given by   
\begin{align*}
L_A & = C_{d,K}\,N\Bigl(2d\,c_{max,1}+\frac{\sigma^2}{2}\Bigr)c_{max,3}^2,
\qquad
L_b = C_{d,K}\,N\Bigl(2d\,c_{max,1}+\frac{\sigma^2}{2}\Bigr)c_{max,4}, \\
C_A & = 4d^2K^2 c_{max,2}^4,
\qquad
C_b = \Bigl(\sigma^4+\frac{16}{T^2}\Bigr)K^2 c_{max,4}^2
\end{align*}
with $C_{d,K}$ denoting a constant depending only on $d$ and $K$.

When the trapezoidal rule (Section~\ref{sec:discretization}) is used in place of the left-endpoint Riemann sum, the discretization bias improves from $O(\Delta t)$ to $O(\Delta t^2)$. 
% The $C_b^6$ regularity ensures that the maps $t\mapsto \mathbf G(t)$ and $t\mapsto \mathbf d(t)$ in the proof of Lemma~\ref{lem:discretization_bound} are $C^2$ in time, so the trapezoidal rule reduces the per-interval quadrature error from $O(\Delta t^2)$ to $O(\Delta t^3)$, yielding an $O(\Delta t^2)$ discretization bias (Lemma~\ref{lem:trap_discretization}). The concentration bounds (Lemma~\ref{lem:concentration}) are unchanged. Figure~\ref{fig:M_scaling}(right) confirms this prediction.

\begin{corollary}[Trapezoidal error bound]\label{cor:trapezoid_bound}
Under Assumption~{\rm\ref{ass:parametric_theory}} with the strengthened regularity $\psi_k, \phi_\ell^{sym} \in C_b^6(\R^d)$, define the trapezoidal estimator $\widehat\theta_{M,\Delta t}^{\,\mathrm{trap}}$ by replacing the left-endpoint sums in~\eqref{eq:normal_matrix_vector} with their trapezoidal counterparts:
\begin{equation*}%\label{eq:trap_normal}
    \begin{aligned}
    \mathbf{A}_{M,\Delta t}^{\mathrm{trap}}  & = \frac{1}{2MLN} \sum_{m,\ell}  \sum_i \Big[(\mathbf{F}_{i,t_\ell}^{m})^\top \mathbf{F}_{i,t_\ell}^{m} + (\mathbf{F}_{i,t_{\ell+1}}^{m})^\top \mathbf{F}_{i,t_{\ell+1}}^{m}\Big], \quad \\
     \mathbf{b}_{M,\Delta t}^{\mathrm{trap}} & =
    \frac{1}{MT} \sum_{m,\ell}\Big[ \frac{\sigma^2}{4} (\boldsymbol{\delta}_{t_\ell}^{m} + \boldsymbol{\delta}_{t_{\ell+1}}^{m}) \Delta t- (\mathbf{h}_{t_{\ell+1}}^{m} - \mathbf{h}_{t_\ell}^{m}) \Big].
    \end{aligned}
\end{equation*}
Then, for every $\eta\in(0,1)$, there exist constants $C$, $M_0$ and $\Delta_0$ such that
\begin{equation}\label{eq:trap_error}
    |\widehat\theta_{M,\Delta t}^{\,\mathrm{trap}}-\theta_*|
    \le
    C\!\left((\Delta t)^2 + \frac{1}{\sqrt{M\eta}}\right)
\end{equation}
with probability at least $1-\eta$ for all $M\ge M_0/\eta$ and $\Delta t\le \Delta_0$.
\end{corollary}

We postpone the proofs of the main results to Appendix~\ref{app:proofs_parametric} and give a sketch of the main ideas here.  To establish the error bound, the first task is to show that the \emph{continuum expectation normal equation} $\mathbf{A}_*\theta = \mathbf{b}_*$, where $\mathbf{A}_*= \lim_{M\to\infty, \Delta t\to 0} \mathbf{A}_{M,\Delta t}$ and $\mathbf{b}_*= \lim_{M\to\infty, \Delta t\to 0} \mathbf{b}_{M,\Delta t}$, has a unique solution at $\theta_* = \mathbf{A}_*^{-1}\mathbf{b}_*$. Here, a key step is to show that $\mathbf{A}_*$ is invertible, for which we introduce a coercivity condition. The second task is to quantify the distances between the empirical normal matrix/vector and the continuum expectation normal matrix/vector, and establish an error bound for the estimator $\widehat \theta_{M,\Delta t}$.

%%%%%%%%%%%%%%
\section{Numerical experiments}\label{sec:num}

This section examines the performance of the proposed self-test approach on synthetic data. We first validate the error bound in Theorem~\ref{thm:error_bound} and the velocity-free advantage of the self-test on a reference model that satisfies all assumptions (Sections~\ref{sec:setup}--\ref{sec:nn_results}). We then probe the boundaries of the theory by testing models that violate or stress different assumptions (Section~\ref{sec:boundary_tests}).

\subsection{Experimental setup}\label{sec:setup}
\label{sec:numerical_design}%
% We test the error bound on synthetic data where ground truth potentials are known.
% All experiments share the following data generation and evaluation protocol.

\paragraph{Data generation.} Ground truth trajectories are generated using the Euler--Maruyama scheme with total time $T = 1.0$, diffusion coefficient $\sigma = 1.0$, and unless otherwise specified, $N = 10$ particles in $d = 2$ dimensions.
We simulate $M = 20{,}000$ independent ensembles and generate observations under two protocols that differ in the relationship between the simulation step~$\delta t$ and the observation interval~$\Delta t_{\rm obs}$. Initial conditions are drawn as $X_0^{i,m} \sim \mathcal{N}(0, 0.25 \cdot I_d)$.
At each observation time, particle indices are randomly permuted to remove trajectory information. 

In the \emph{gap regime} ($\delta t \ll \Delta t_{\rm obs}$), a single set of $M$ fine-step trajectories is simulated with step $\delta t = 10^{-4}$ (random seed fixed across all $\Delta t_{\rm obs}$), and particle positions are recorded at intervals $\Delta t_{\rm obs} \in \{10^{-3}, 10^{-2}, 10^{-1}\}$ by striding the same fine-step trajectory at stride $\Delta t_{\rm obs}/\delta t$; this ensures that all observation frequencies share identical underlying dynamics and differ only in sampling rate. This is the default protocol for the method comparison, boundary tests, and the trapezoid-quadrature $M$-scaling study (Sections~\ref{sec:model_a_validation}--\ref{sec:boundary_tests}, Figure~\ref{fig:M_scaling}).

Additionally, to test the discrete-time model discussed in Section \ref{sec:discreteModel}, we consider a \emph{zero-gap regime} ($\delta t = \Delta t_{\rm obs}$), in which the data is simulated directly at each observation time step $\Delta t_{\rm obs} \in \{10^{-4}, 10^{-3}, 10^{-2}, 10^{-1}\}$. This regime is used for the $M$-scaling analysis in Figure~\ref{fig:discrete_bias}.

\paragraph{Evaluation metrics.} 
We evaluate accuracy using the relative $L^2(\rho)$ error on the gradients:
\[
\text{Err}(\nabla V) = \frac{\| \nabla \hat{V} - \nabla V^* \|_{L^2(\rho_V)}}{\| \nabla V^* \|_{L^2(\rho_V)}}, \quad
\text{Err}(\nabla \Phi) = \frac{\| \nabla \hat{\Phi} - \nabla \Phi^* \|_{L^2(\rho_\Phi)}}{\| \nabla \Phi^* \|_{L^2(\rho_\Phi)}},
\]
where, for radial models, $\rho_V$ is the one-dimensional empirical density of particle radial positions $r_i = |X_i| \in \R^+$ and $\rho_\Phi$ is the density of pairwise distances $r_{ij} = |X_i - X_j| \in \R^+$, both estimated via Gaussian KDE (bandwidth $h = 0.15\,\hat\sigma_{\rm data}$) on the full $M = 20{,}000$ dataset at $\Delta t_{\rm obs} = 10^{-3}$ ($2 \times 10^8$ particle positions for $\rho_V$, $9 \times 10^8$ pairwise distances for $\rho_\Phi$). Since $\rho_V$ and $\rho_\Phi$ are properties of the stationary distribution and do not depend on the observation frequency, this single KDE is used for all $\Delta t_{\rm obs}$ settings. For non-radial models, $\rho_V$ and $\rho_\Phi$ are the corresponding densities on $\R^d$.
The $L^2(\rho)$ integrals are computed via trapezoidal quadrature on a fixed grid of 2000 points over a model-specific range that covers the support of the empirical density.
All reported errors are mean $\pm$ std over 10 non-overlapping blocks of $M = 2{,}000$ from the $M = 20{,}000$ pool, each solved independently against the shared KDE density.

\subsection{Validation of the error bound and scaling laws}\label{sec:model_a_validation}

\paragraph{Reference model.}
We validate the error bound in Theorem~\ref{thm:error_bound} on a reference model that satisfies all assumptions of the theory, including $C^4_b$ regularity, finite basis moment bounds, and a well-conditioned Gram matrix.
The reference model has radial potentials of the form 
\begin{equation}\label{eq:ref_model}
V^*(x) = \tfrac{1}{2}\alpha_1 |x| + \alpha_2 |x|^2, \qquad
\Phi^*(r) = \sum_{k=1}^{2} \beta_k \exp\!\Bigl(-\frac{(r - c_k)^2}{2\sigma_k^2}\Bigr),
\end{equation}
with parameters $\alpha = (-1, 2)$, $\beta = (-3, 2)$. Here, $c = (0.75, 1.5)$, and $(\sigma_1,\sigma_2) = (0.125, 0.25)$.
The confining potential $V^*$ combines linear and quadratic radial terms.
Since $\Phi^* \in C^\infty(\R)$ with rapid Gaussian decay, the basis moment bounds $c_{max,p}$ are finite for all~$p$; its well-separated peaks yield a well-conditioned $\mathbf{G}_\Phi$.
The oracle basis consists of $K_V = 2$ functions ($|x|$ and $|x|^2$) for $V$ and $K_\Phi = 2$ Gaussian bumps for $\Phi$, giving total dimension $K = 4$.
We use this oracle basis to eliminate approximation error; all regularized solvers select $\lambda$ by Hansen's L-curve method.

\paragraph{Convergence rates in sample size.}
Figure~\ref{fig:M_scaling} shows the convergence rates in sample size for the left-endpoint Riemann sum and the trapezoidal rule on the reference model.
\begin{figure}[ht]
\centering
\includegraphics[width=\textwidth]{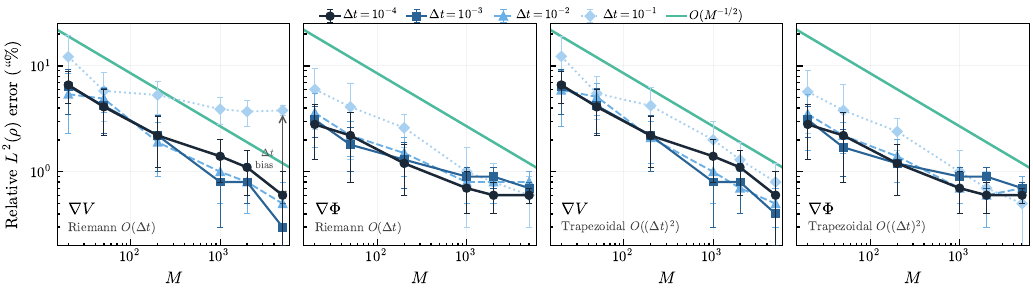}
\caption{$M$-scaling on the reference model under Riemann-sum (left pair) and trapezoidal (right pair) time integrations. Both with data generated with $\delta t = 10^{-4}$ and integrated with the various $\Delta t$ values.
The Riemann-sum has error bound $O(\Delta t + M^{-1/2})$, so the four $\Delta t$ values track the $O(M^{-1/2})$ rate (green line) until they saturate at an $O(\Delta t)$ floor (left pair), while the trapezoidal rule has error bound $O((\Delta t)^2 + M^{-1/2})$, so all four $\Delta t$ values track the $O(M^{-1/2})$ rate without saturation (right pair).
%  At $\Delta t = 10^{-1}$ the $\nabla V$ error saturates at ${\sim}3.8\%$ (arrow), an $O(\Delta t)$ discretization floor consistent with Theorem~\ref{thm:error_bound}.
% Right pair: trapezoidal rule.  All four $\Delta t$ values track the $O(M^{-1/2})$ rate (green line), consistent with the $O((\Delta t)^2 + M^{-1/2})$ bound in Corollary~\ref{cor:trapezoid_bound}.
The mean and standard deviation are computed over 10 trials per point; detailed numbers in Table~\ref{tab:M_scaling}.
}
\label{fig:M_scaling}
\end{figure}
The $O(\Delta t)$ discretization floor visible at $\Delta t = 10^{-1}$ (left pair) arises because the $O(\Delta t^2)$ per-interval quadrature error (Section~\ref{sec:discretization}) accumulates to $O(\Delta t)$ over $L = T/\Delta t$ intervals, producing a parameter bias that persists as $M$ increases. The trapezoidal rule lowers this floor to $O(\Delta t^2)$ (Corollary~\ref{cor:trapezoid_bound}, right pair): coarse observations ($\Delta t = 10^{-1}$) now achieve the same convergence as dense ones since $(\Delta t)^2 $ is smaller than $M^{-1/2}$ for all $M$ values in the plot. The numbers are reported in Table~\ref{tab:M_scaling} in Appendix~\ref{app:detailed_results}. 

\paragraph{Discrete-model bias.}
When $\delta t = \Delta t$ (zero gap), the particle system is a discrete time series rather than an SDE discretization, so the $O(\Delta t)$ bias is intrinsic to the model and cannot be removed by refining the quadrature; see Section~\ref{sec:discretization}. The saturation floors in Figure~\ref{fig:discrete_bias} are predicted by the leading-order $O(\Delta t)$. Since this model bias dominates the quadrature correction, the Riemann-sum and trapezoidal rules yield identical results.

\begin{figure}[ht]
\centering
\includegraphics[width=0.55\textwidth]{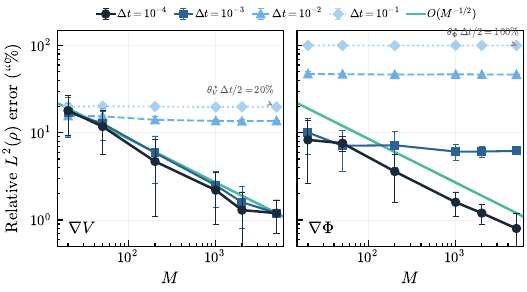}
\caption{Convergence for discrete-time model (i.e., $\delta t = \Delta t$, zero gap). The intrinsic $O(\Delta t)$ bias of the discrete-time model (see Section~\ref{sec:discretization}) dominates the error when $\Delta t$ is large, and the convergence in $M$ is only visible at $\Delta t \le 10^{-3}$ (and $\Delta t= 10^{-4}$ for $\Phi$) where the bias is small enough to allow the $O(M^{-1/2})$ statistical error to dominate. See also in Table~\ref{tab:M_scaling}.
%At $\Delta t \ge 10^{-2}$ the errors saturate at floors predicted by $\theta^*\Delta t / 2$: ${\sim}20\%$ for $\nabla V$ and ${\sim}100\%$ for $\nabla\Phi$ at $\Delta t = 10^{-1}$.
%At $\Delta t \le 10^{-3}$ the errors track the $O(M^{-1/2})$ statistical rate (green line).
% 10 trials per point.
}
\label{fig:discrete_bias}
\end{figure}

\paragraph{Condition number scaling.}
Figure~\ref{fig:cond_numbers} shows the condition numbers $\kappa_*$, $\kappa_{VV}$, $\kappa_{\Phi\Phi}$ of the normal matrix $\mathbf{A}_*$ and its diagonal blocks as $N$ varies at $d = 2, 5, 10$ (with $M=2000$, $\Delta t = 10^{-2}$). Proposition~\ref{prop:joint_coercivity} predicts $\kappa_{VV} = O(1)$, $\kappa_{\Phi\Phi} = O(N)$ (since $c_N = \frac{N-1}{N^2}\sim 1/N$), and $\kappa_* = O(N)$; these rates are confirmed at $d=10$. 
The full-matrix $\kappa_*$ exceeds both block values because the off-diagonal $V$--$\Phi$ coupling shrinks $\lambda_{\min}(\mathbf{A}_*)$; see Table~\ref{tab:cond_number} (Appendix~\ref{app:detailed_results}) for the full $N \times d$ grid.
Further tests with $d=20$ (not presented here) show that $\kappa_{VV}$ remains stable, but $\kappa_{\Phi\Phi}$ and hence $\kappa_*$ exceed $10^6$ at $N=100$ due to the concentration of pairwise distances.
Thus, the flattening of $\kappa_{\Phi\Phi}$ at $d=2$ and $d=5$ is likely due to a lack of proper concentration of the pairwise distances which prevents the $\Phi$-block conditioning from seeing the full $O(N)$ growth,  particularly when the basis functions are smooth and slowly varying relative to the distance distribution. 
% Table~\ref{tab:cond_number} (Appendix~\ref{app:detailed_results}) lists the full $N \times d$ condition-number grid.
\begin{figure}[ht!]
\centering
\includegraphics[width=\textwidth]{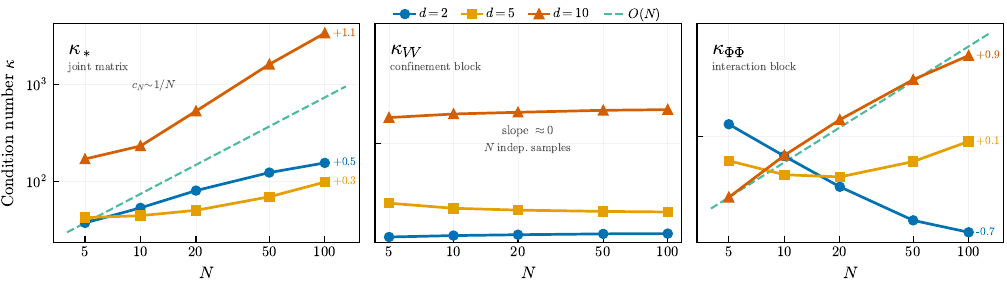}
\caption{Condition numbers scaling with~$N$ for the normal matrix and its diagonal blocks~($\kappa_*$, $\kappa_{VV}$, $\kappa_{\Phi\Phi}$).
Slopes (italic numbers at $N{=}100$) are log-log regression rates from Table~\ref{tab:cond_number}.
Dashed/dotted green lines show the $O(N)$ theoretical rates.
Left: $\kappa_*$ increases in $N$, with a rate near $O(N)$ when $d=10$.
Center:~$\kappa_{VV}$ is $N$-independent.  
Right:~$\kappa_{\Phi\Phi}$ has a rate near $O(N)$ when $d=10$ due to distance concentration.}
\label{fig:cond_numbers}
\end{figure}

\subsection{Method comparison}\label{sec:oracle_comparison}

We compare the four estimation strategies in Section \ref{sec:baselines} (labeled MLE, Sinkhorn MLE, self-test LSE, and self-test NN) as $\Delta t$
varies over $\{10^{-4}, 10^{-3}, 10^{-2}, 10^{-1}\}$ with
$\delta t = 10^{-4}$ fixed.
Labeled MLE and Sinkhorn MLE require the finite-difference approximated velocity $v_i = \Delta X_i / \Delta t$,
whose $O(\Delta t)$ discretization bias persists as $M \to \infty$.
The self-test loss~\eqref{eq:loss_trajFree} depends only on particle positions
and requires no velocity estimation, so this bias does not arise.

\begin{table}[ht]
\centering
\caption{Method comparison across observation intervals~$\Delta t$ for the reference model (relative gradient errors, \%).
Oracle: true basis, mean $\pm$ std over 10 trials.
NN: MLP $[64{,}64{,}64]$, mean $\pm$ std over 10 trials.
{\best{Green}} = best per $\Delta t$.}
\label{tab:dt_obs_sweep}
\vspace{2mm}
\resizebox{\textwidth}{!}{%
\setlength{\tabcolsep}{3pt}
\begin{tabular}{c cc cc cc cc}
\toprule
\multicolumn{9}{c}{\scriptsize $d{=}2$,\; $N{=}10$,\; $M{=}2{,}000$,\; $T{=}1$,\; $\sigma{=}1$,\; $\delta t{=}10^{-4}$} \\
\midrule
& \multicolumn{2}{c}{Labeled MLE}
& \multicolumn{2}{c}{Self-Test LSE}
& \multicolumn{2}{c}{Sinkhorn MLE}
& \multicolumn{2}{c}{Self-Test NN} \\
\cmidrule(lr){2-3}\cmidrule(lr){4-5}\cmidrule(lr){6-7}\cmidrule(lr){8-9}
$\Delta t$
  & $\nabla V$ & $\nabla \Phi$
  & $\nabla V$ & $\nabla \Phi$
  & $\nabla V$ & $\nabla \Phi$
  & $\nabla V$ & $\nabla \Phi$ \\
\midrule
$10^{-4}$ & \best{$0.80{\pm}0.55$} & \best{$0.35{\pm}0.19$}
           & $1.35{\pm}0.79$ & $1.24{\pm}0.29$
           & $1.34{\pm}0.80$ & $0.82{\pm}0.18$
           & --- & --- \\
$10^{-3}$ & $1.41{\pm}0.45$ & $5.18{\pm}0.16$
           & \best{$0.67{\pm}0.37$} & \best{$0.74{\pm}0.15$}
           & $3.13{\pm}0.58$ & $7.17{\pm}0.15$
           & $0.91{\pm}0.17$ & $1.74{\pm}0.13$ \\
$10^{-2}$ & $10.44{\pm}0.41$ & $37.49{\pm}0.28$
           & \best{$0.80{\pm}0.43$} & \best{$1.10{\pm}0.17$}
           & $21.64{\pm}0.34$ & $47.38{\pm}0.25$
           & $1.25{\pm}0.33$ & $2.47{\pm}0.22$ \\
$10^{-1}$ & $11.30{\pm}0.58$ & $89.38{\pm}0.18$
           & $6.84{\pm}0.58$ & $6.93{\pm}0.72$
           & $45.37{\pm}0.56$ & $95.78{\pm}0.23$
           & \best{$3.15{\pm}0.68$} & \best{$5.80{\pm}0.42$} \\
\bottomrule
\end{tabular}}
\end{table}

Table~\ref{tab:dt_obs_sweep} reveals a \emph{crossover} between MLE and self-test as $\Delta t$ increases. When $\Delta t = \delta t$, there is no discretization error and the sampling error caused by the noise $\sigma/\sqrt{\Delta t}$ is averaged out by the $M\times L$ observations, so MLE achieves the lowest errors. As $\Delta t$ grows, the velocity bias degrades MLE, and the self-test overtakes MLE at $\Delta t = 10^{-2}$ and $\Delta t = 10^{-3}$. At $\Delta t = 10^{-2}$, the self-test errors are more than an order of magnitude lower than MLE for both $\nabla V$ and $\nabla \Phi$; at $\Delta t = 10^{-1}$, the same order-of-magnitude advantage remains for $\nabla \Phi$. This improvement comes from avoiding velocity estimation entirely: the self-test loss only incurs quadrature error in the dissipation and diffusion integrals, while the energy change $\delta E_{f}$ is computed exactly from data.

The Sinkhorn MLE method, which also relies on velocity estimation, performs worse than MLE at every $\Delta t$ due to the additional label-matching error, which degrades significantly with $\Delta t$ as nonlinear non-Gaussian dynamics prevents accurate label recovery. 

In summary, as a practical method, the self-test LSE has a significant advantage over Sinkhorn MLE at every $\Delta t>\delta t$. In particular, recall that the self-test approach is much more computationally efficient than Sinkhorn MLE since it does not require the costly label-matching step, so the self-test is more accurate and more efficient than Sinkhorn MLE.  

%%%%%%%%%%% ======== 
\subsection{Neural network estimation}\label{sec:nn_results}

The self-test NN requires neither particle labels nor basis knowledge, so it is applicable to general problems with no prior information about the underlying potentials, particularly when the potentials are non-radial or high-dimensional. 

We test self-test NN on the reference model and a non-radial model, and compare it with the three methods with basis functions and with the radial basis function (RBF) method. Following Algorithm~\ref{alg:nn_self_test}, we use 3-layer MLPs ($[64{,}64{,}64]$, Softplus, ${\sim}17{,}000$ parameters), and train it by minimizing the self-test loss. Training details, including asymmetric learning rates, cosine annealing, gradient clipping, and mini-batch construction, are given in Appendix~\ref{app:training_protocol}.

\paragraph{Radial potentials.} For the radial potentials in the reference model, the self-test NN achieves errors comparable to or better than the parametric methods (Table~\ref{tab:dt_obs_sweep}). In particular, at the large $\Delta t = 10^{-1}$, the self-test NN outperforms all parametric methods because neural networks can learn a more flexible representation than when the multi-step iterations of the nonlinear and non-Gaussian dynamics render the basis functions inadequate.

\phantomsection\label{sec:non_radial}%
\paragraph{Non-radial potentials.} For non-radial potentials, no finite radial basis can represent the target, so parametric and RBF methods do not apply.
The self-test NN extends directly: we replace the radial MLP with a general MLP enforcing evenness $\Phi(z) = \Phi(-z)$, without changing the loss function---only the architecture changes, unlike basis methods which would require entirely new function families.
Because the self-test loss evaluates Laplacian terms $\Delta V$ and $\Delta\Phi$, all activations must be $C^2$-smooth; ReLU and Leaky~ReLU are excluded since their second derivatives vanish almost everywhere (Appendix~\ref{app:architecture}).

We test this on anisotropic potentials: $V(x) = a_1 x_1^2 + a_2 x_2^2$ with $a = (1,4)$, and $\Phi(z) = A\exp\bigl(-\frac{1}{2}(z_1^2/s_1^2 + z_2^2/s_2^2)\bigr)$ with $A = 2$ and $s = (0.5, 1.5)$.
Unlike radial models where errors are measured on the scalar derivative $d\Phi/dr$, non-radial evaluation compares the full vector gradient $\nabla_z \Phi \in \R^d$. Figure~\ref{fig:nn_model_aniso} shows the true and recovered potentials in the case $d = 2$, where the NN is trained with data using $\Delta t_{\mathrm{obs}} = \Delta t_{\mathrm{fine}} = 10^{-3}$ (zero discretization gap) and $M = 2000$. The estimator captures the anisotropic contours of $V$ and $\Phi$, with relative errors of $3.6\%$ for $\nabla V$ and $12.3\%$ for $\nabla \Phi$. This demonstrates that the self-test NN can successfully recover non-radial potentials. 
% see Appendix~\ref{app:dt_obs_sensitivity} for additional non-radial experiments.

\begin{figure}[ht]
\centering
\includegraphics[width=\textwidth]{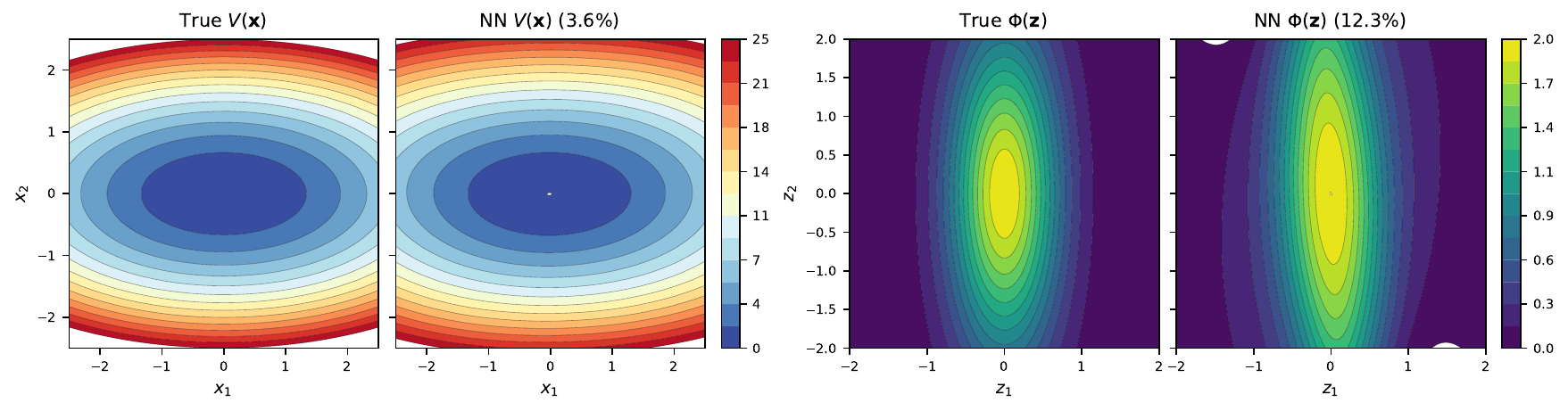}
\caption{Non-radial potential recovery ($d{=}2$). 
% Left pair: anisotropic confinement $V(x)$ with $a=(1,4)$. Right pair: anisotropic Gaussian interaction $\Phi(z)$ with $s=(0.5,1.5)$. True potentials (left of each pair) vs.\ NN recovery (right). 
Percentages are relative $L^2(\rho)$ errors.}
\label{fig:nn_model_aniso}
\end{figure}

%%%%%%%%%%%%%%
\subsection{Boundary stress tests}\label{sec:boundary_tests}
We further test the self-test approach on four additional models at the boundaries of the theory, each designed to violate or stress a condition in Assumption~\ref{ass:parametric_theory}. 
The potentials in these four models are as follows. 
\begin{itemize}
    \item \emph{Smoothness test:}
    $V^*(x) = \frac{1}{2} \alpha_1 |x| + \alpha_2 |x|^2$, \;
    $\Phi^*(r) = \beta_1 \mathbf{1}_{[0.5,1]}^\epsilon (r)+ \beta_2 \mathbf{1}_{[1,2]}^\epsilon (r)$, 
    with $\alpha=(-1,2), \beta=(-3,2)$. Here, $\mathbf{1}_{[0.5,1]}^\epsilon $ indicates the smoothed indicator interaction via tanh smoothing with $\varepsilon{=}0.05$, restoring $C^4$ regularity but with relatively large derivatives.

    \item \emph{Conditioning test:}
    $V^*(x) = \frac{1}{4}(|x|^2 - 1)^2 - \frac{1}{4}$, \;
    $\Phi^*(r) = -\gamma r/(r + 1)$,
    with $\gamma=0.5$. The slow-decay interaction yields ill-conditioned normal matrices with condition number $\kappa$ three orders of magnitude above the reference model; see Table~\ref{tab:cond_all}.  
%  the ill-conditioned Gram matrix amplifies regularization error in the self-test normal equations. Self-test still recovers $\nabla V$ at large~$\Delta t$ because it avoids velocity estimation   (Section~\ref{sec:oracle_comparison}).

    \item \emph{Singularity test:}
    $V^*(x) = \frac{k}{2} |x|^2$, \;
    $\Phi^*(r) = 4\varepsilon\bigl[(\sigma/r)^{12} - (\sigma/r)^{6}\bigr]$,
    with $k=2, \varepsilon=0.5, \sigma=0.5$. We clamp pairwise distances at $r_{\text{safe}}=0.35$ to avoid numerical instability, but the resulting large derivatives still pose a challenge, with condition number $\kappa$ two orders of magnitude above the reference model; see Table~\ref{tab:cond_all}.

    \item \emph{Smooth control:}
    $V^*(x) = \frac{1}{4}(|x|^2 - 1)^2 - \frac{1}{4}$, \;
    $\Phi^*(r) = D\bigl[(1 - e^{-a(r - r_0)})^2 - (1 - e^{ar_0})^2\bigr]$,
    with $D=0.5, a=2, r_0=0.8$. This model has a double-well confining potential and a smooth Morse interaction potential with moderate derivatives, so it satisfies all assumptions and serves as a well-behaved baseline for comparison.
\end{itemize}
Table~\ref{tab:boundary_models} presents an overview of these models and the specific assumptions they violate or stress.

\begin{table}[ht]
\centering
\caption{Boundary test models and the assumptions they violate or stress.}
\label{tab:boundary_models}
% \vspace{2mm}
\small
\begin{tabular}{l l l}
\toprule
Model & $\Phi$ type & Violation / stress point \\
\midrule
Smoothness test & Smoothed indicator $\mathbf{1}_{[a,b]}^\epsilon$ & $C^4_b$ smoothness (Assumption~\ref{ass:parametric_theory}(i)) \\
Conditioning test & $\gamma/(r+1)$ inverse & Coercivity --- near-singular Gram matrix \\
Singularity test & $r^{-12} - r^{-6}$ singularity & $C^4_b$ + gradient blowup \\
Smooth control & $D(1-e^{-a(r-r_0)})^2$ & None (well-behaved baseline) \\
\bottomrule
\end{tabular}
\end{table}

\begin{figure}[ht]
\centering
\includegraphics[width=\textwidth]{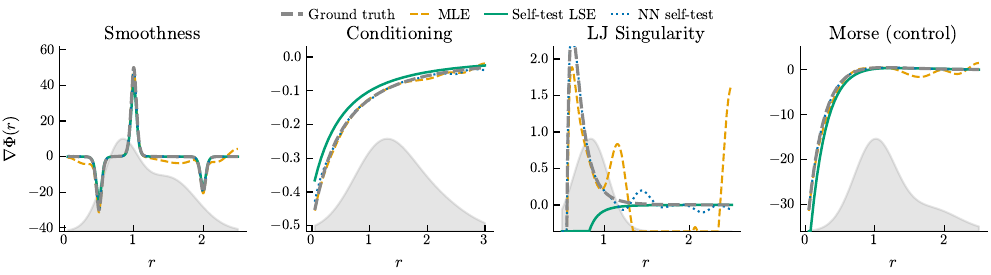}
\caption{Typical estimators of $\Phi'(r)$ for the four boundary test models ($d{=}2$, $M{=}2{,}000$, $\sigma{=}1$, $\delta t{=}10^{-4}$, $\Delta t{=}10^{-2}$).
Gray shading indicates the approximate exploration measure $\rho_{\Phi}$.
\textbf{Smoothness:} Smoothed indicator with sharp transitions at $r = 0.5$ and $r = 1$.
\textbf{Conditioning:} Slow-decaying $1/(r{+}1)$ interaction that challenges conditioning.
\textbf{Singularity:} Lennard-Jones $r^{-12}{-}r^{-6}$ singularity; gradient diverges as $r \to 0$.
\textbf{Smooth control:} Morse potential with equilibrium at $r_0 = 0.8$.
}
\label{fig:boundary_recovery}
\end{figure}

\paragraph{Stability across $\Delta t$.} We examine stability across~$\Delta t$ of the self-test method in these boundary test models. Table~\ref{tab:boundary_results} reports the relative errors. The self-test estimators consistently outperform the Sinkhorn-MLEs across $\Delta t$ for all four models. At $\Delta t = 10^{-3}$, the self-test LSE is competitive with labeled MLE and is substantially better on the singularity and smooth-control models. As $\Delta t$ increases, the self-test estimators degrade much slower than the MLEs. At $\Delta t = 10^{-1}$, the self-test NN is best on the smoothness, singularity, and smooth-control models and yields the best $\nabla \Phi$ error on the conditioning model, demonstrating robustness to coarse observations even in these challenging settings.

\begin{table}[ht]
\centering
\caption{Relative errors (\%) as $\Delta t$ varies in boundary tests. Parametric estimators: mean $\pm$ std in 10 trials ($M{=}2{,}000$ subsampled from $M_{\text{total}}{=}20{,}000$). {\best{Green}} = best among unlabeled methods.}
\label{tab:boundary_results}
\vspace{2mm}
\scriptsize
\setlength{\tabcolsep}{3pt}
\begin{tabular}{cl cc cc cc cc}
\toprule
\multicolumn{10}{c}{\scriptsize $d{=}2$,\; $N{=}10$,\; $M_{\text{total}}{=}20{,}000$,\; $M{=}2{,}000$,\; $T{=}1$,\; $\sigma{=}1$,\; $\delta t{=}10^{-4}$} \\
\midrule
& & \multicolumn{2}{c}{Labeled MLE} & \multicolumn{2}{c}{Self-Test} & \multicolumn{2}{c}{Sinkhorn} & \multicolumn{2}{c}{Self-Test NN} \\
\cmidrule(lr){3-4}\cmidrule(lr){5-6}\cmidrule(lr){7-8}\cmidrule(lr){9-10}
Model & $\Delta t$ & $\nabla V$ & $\nabla \Phi$ & $\nabla V$ & $\nabla \Phi$ & $\nabla V$ & $\nabla \Phi$ & $\nabla V$ & $\nabla \Phi$ \\
\midrule
Smoothness & $10^{-3}$ & $6.70{\pm}0.62$ & $12.98{\pm}0.19$ & $\best{1.07{\pm}0.52}$ & $\best{2.01{\pm}0.20}$ & $6.30{\pm}0.53$ & $14.42{\pm}0.19$ & $1.76{\pm}0.54$ & $3.32{\pm}0.29$ \\
  & $10^{-2}$ & $21.05{\pm}0.55$ & $52.63{\pm}0.18$ & $\best{0.85{\pm}0.47}$ & $\best{2.55{\pm}0.29}$ & $18.23{\pm}0.45$ & $59.31{\pm}0.16$ & $1.41{\pm}0.44$ & $4.60{\pm}0.86$ \\
  & $10^{-1}$ & $3.57{\pm}0.32$ & $92.15{\pm}0.22$ & $9.07{\pm}1.33$ & $10.01{\pm}0.74$ & $31.32{\pm}0.20$ & $95.89{\pm}0.12$ & $\best{3.64{\pm}1.42}$ & $\best{8.77{\pm}0.88}$ \\
\midrule
Conditioning & $10^{-3}$ & $1.15{\pm}0.59$ & $11.97{\pm}9.08$ & $\best{1.46{\pm}0.73}$ & $31.02{\pm}9.63$ & $13.58{\pm}1.06$ & $271{\pm}11$ & $3.83{\pm}0.38$ & $\best{26.7{\pm}8.8}$ \\
  & $10^{-2}$ & $2.67{\pm}0.82$ & $15.85{\pm}11.25$ & $\best{1.58{\pm}0.55}$ & $40.26{\pm}10.29$ & $34.74{\pm}1.22$ & $719{\pm}14$ & $7.84{\pm}3.51$ & $\best{24.6{\pm}11.7}$ \\
  & $10^{-1}$ & $21.58{\pm}0.75$ & $24.75{\pm}12.55$ & $\best{6.19{\pm}0.85}$ & $185{\pm}31$ & $59.21{\pm}1.18$ & $1258{\pm}13$ & $10.8{\pm}7.3$ & $\best{125{\pm}12}$ \\
\midrule
Singularity & $10^{-3}$ & $25.46{\pm}0.53$ & $55.12{\pm}0.09$ & $8.88{\pm}0.46$ & $\best{13.89{\pm}0.25}$ & $25.46{\pm}0.53$ & $55.12{\pm}0.09$ & $\best{1.91{\pm}0.81}$ & $40.4{\pm}18.6$ \\
  & $10^{-2}$ & $4.53{\pm}0.37$ & $99.40{\pm}0.10$ & $13.34{\pm}0.85$ & $21.75{\pm}0.91$ & $4.53{\pm}0.37$ & $99.40{\pm}0.10$ & $\best{2.70{\pm}1.16}$ & $\best{4.96{\pm}0.41}$ \\
  & $10^{-1}$ & $47.51{\pm}0.21$ & $100{\pm}0.08$ & $15.36{\pm}2.19$ & $39.68{\pm}2.13$ & $49.72{\pm}0.20$ & $101{\pm}0.08$ & $\best{1.60{\pm}0.60}$ & $\best{4.80{\pm}0.89}$ \\
\midrule
Smooth ctrl & $10^{-3}$ & $2.77{\pm}1.05$ & $5.61{\pm}0.91$ & $\best{1.72{\pm}0.88}$ & $\best{1.69{\pm}0.73}$ & $4.82{\pm}1.17$ & $8.87{\pm}0.85$ & $2.34{\pm}0.29$ & $3.43{\pm}0.82$ \\
  & $10^{-2}$ & $13.06{\pm}1.02$ & $31.30{\pm}0.71$ & $\best{2.01{\pm}1.25}$ & $3.86{\pm}1.51$ & $4.13{\pm}0.98$ & $8.00{\pm}0.75$ & $2.96{\pm}0.64$ & $\best{2.95{\pm}0.38}$ \\
  & $10^{-1}$ & $32.19{\pm}0.47$ & $78.79{\pm}1.09$ & $8.05{\pm}5.73$ & $33.42{\pm}10.24$ & $33.88{\pm}0.63$ & $38.51{\pm}0.56$ & $\best{5.91{\pm}0.33}$ & $\best{6.78{\pm}1.01}$ \\
\bottomrule
\end{tabular}
\end{table}

\begin{table}[ht]
\centering
\caption{Condition number $\kappa(\mathbf{A}_*)$ across all models.
Rate in $N$: log-log regression $\partial\!\log\kappa / \partial\!\log N$;
Proposition~\ref{prop:joint_coercivity} predicts slope~${\le}\,1$.}
\label{tab:cond_all}
\vspace{2mm}
\small
\setlength{\tabcolsep}{4pt}
\begin{tabular}{l rrrrr r}
\toprule
\multicolumn{7}{c}{\scriptsize $d{=}2$,\; $\Delta t{=}10^{-2}$,\; $M{=}2000$,\; oracle basis} \\
\midrule
Model & $N{=}5$ & $N{=}10$ & $N{=}20$ & $N{=}50$ & $N{=}100$ & Rate in $N$ \\
\midrule
Reference & 37     & 53      & 80     & 123     & 155      & \best{$0.5$} \\
Smoothness & 42     & 46      & 88     & 208     & 396      & $0.8$ \\
Conditioning & 81K    & 146K    & 262K   & 564K    & 962K     & $0.8$ \\
Singularity & 10.8K  & 12.1K   & 11.8K  & 12.1K   & 11.8K    & $0.0$ \emph{(constant)} \\
Smooth ctrl & 7.9K & 14.2K & 25.4K  & 53.5K   & 89.5K    & $0.8$ \\
\bottomrule
\end{tabular}
\end{table}

%%%%%%%%%%%%%%%%=========
We also test the RBF method on the radial models with a dictionary of 20 Gaussian RBFs with centers uniformly spaced in $[0, 3]$ and bandwidth $0.5$; see Appendix~\ref{app:implementation} for details. Figure~\ref{fig:all_methods} there shows that with such limited basis, the RBF method leads to oscillatory estimators, underperforming the self-test estimators.

%%%%%%%%%%%%%%%%=========
\section{Conclusion}\label{sec:conclusion}
We have developed a trajectory-free loss function for learning interaction and confining potentials from unlabeled particle snapshots. Our method constructs a quadratic loss from the weak-form PDE of the empirical distribution, using data-dependent test functions of the form $f = V + \Phi * \mu_t^N$,  where $\mu_t^N$ is the empirical measure from data. This is fundamentally different from existing approaches: trajectory-based methods require label recovery for velocity estimation (which fails for large $\Delta t$), while optimal transport methods are nonlinear and computationally expensive. 

We have provided nonasymptotic error bounds for the parametric estimator, showing that the error scales as $O((\Delta t)^\alpha + M^{-1/2})$ and is tight, with $\alpha=1$ for the Riemann sum and $\alpha=2$ for the trapezoidal rule. Numerical tests confirm these error bounds. Importantly, the self-test approach has a significant advantage over labeled MLE when the observation gap is large, and it is more accurate and more efficient than Sinkhorn-MLE because it does not require costly label-matching. Furthermore, the self-test approach applies to neural networks, which require no basis knowledge, successfully recovers non-radial potentials, and remains effective at coarse observation intervals where fixed basis choices become limiting.

The method has various potential applications to real-world datasets in physics, biology and social science. The self-test approach enables learning interaction potentials from static snapshots of particle systems, addressing a fundamental challenge across multiple scientific domains. The method's efficiency and robustness to coarse observations make it a promising direction for datasets where continuous trajectory information is unavailable or unreliable.

On the theoretical side, an important direction is to establish minimax rates for the unlabeled setting and compare them with the known minimax rates for labeled data, which would clarify the fundamental limits of learning without trajectories. 

\paragraph{Limitations.} The self-test approach relies on particle homogeneity, which underpins the use of the empirical distribution to encode unlabeled data. Extension to heterogeneous systems with multiple particle types % \cite{LMT21_JMLR,lang2026interacting} 
from unlabeled data would require type identification and coupled weak-form PDEs, which is a challenging open problem beyond the scope of the current framework. Another limitation is that the current method assumes the interaction potential is sufficiently smooth and free of singularities; extending it to singular potentials would require careful handling of derivative blowup, such as data-adaptive truncation or regularization.

\appendix

\section{Numerical details}\label{app:implementation}
%% ============================================================
%%  Appendix A: Implementation Details
%% ============================================================
\subsection{Implementation details}\label{app:implementation_details}

This appendix specifies the computational details deferred from the main text: 
the RBF basis construction in the parametric baselines 
and the full neural network training recipe.
For singular potentials, pairwise distances are clamped at $r_{\text{safe}} = 0.35$ for Lennard-Jones (truncated at $r_{\text{cut}} = 2.5$), and smoothness-test indicators are smoothed via $\tanh$ transitions ($\varepsilon = 0.05$).

\paragraph{RBF basis construction.} We use Gaussian RBFs $\psi_k(r) = \exp\left(-\frac{(r - c_k)^2}{2w^2}\right)$, where $c_k$ are the centers and $w$ is the width. The centers are placed equidistantly on $[0.01, r_{\max}]$, where $r_{\max}$ is determined from the largest dataset as the 99th percentile of pairwise distances for $\Phi$ and of particle norms for $V$. The width is set to $w = 1.5 \cdot r_{\max} / K$, ensuring sufficient overlap between adjacent RBFs for smooth interpolation. For simplicity, we take $K_V = K_\Phi = 20$ centers for all models; thus, the performance of the RBF basis approach can be further improved by increasing the number of centers and fine tuning their placement and width. Gradients and Laplacians of the basis functions are computed analytically, with gradients $\psi_k'(r) = -\frac{r - c_k}{w^2} \psi_k(r)$ and the Laplacian evaluated via the radial formula $\Delta_x f(x) = g''(r) + \frac{d-1}{r} g'(r)$ for any radial function $f(x) = g(|x|)$.

\paragraph{NN Training protocol.}\label{app:training_protocol}\label{app:architecture}
For radial models, both networks take scalar input $|x|$ (${\sim}8{,}513$ parameters each, ${\sim}17{,}026$ total); for non-radial models, $d$-dimensional input (${\sim}8{,}577$ parameters each at $d{=}2$, ${\sim}17{,}154$ total).
The potentials are identifiable only up to additive constants; the offset cancels in the loss.
We use Adam with:
\begin{itemize}
    \item \textbf{Activation:} The self-test loss involves Laplacian terms $\Delta V$, $\Delta\Phi$, requiring $C^2$-smooth activations; ReLU and Leaky~ReLU are excluded since their second derivatives vanish almost everywhere. In the current implementation we use Softplus for the reference model and Softplus or Tanh for the boundary test models.
    \item \textbf{Learning rates:} $\gamma_V = 10^{-4}$ for the $V$ network and $\gamma_\Phi = 5 \times 10^{-4}$ for the $\Phi$ network. The higher learning rate for $\Phi$ was found empirically to improve convergence in our experiments.
    \item \textbf{Cosine annealing:} Both learning rates follow a cosine schedule from their initial values down to $\eta_{\min} = 0.01 \cdot \eta_0$, i.e., $\gamma_V$ anneals from $10^{-4}$ to $10^{-6}$ and $\gamma_\Phi$ from $5 \times 10^{-4}$ to $5 \times 10^{-6}$ over $T_{\max} = 200$ epochs.
    \item \textbf{Gradient clipping:} Each network's gradients are clipped independently at $\ell^2$-norm $1.0$. Independent clipping prevents the larger-gradient network from dominating parameter updates.
\end{itemize}
Each epoch iterates over all $M(L{-}1)$ snapshot pairs, jointly shuffled into mini-batches; the full dataset tensor is pre-allocated on GPU.
Unlike explicit Tikhonov regularization in the basis solver, the neural network relies on implicit regularization: the finite-width MLP restricts the hypothesis space to smooth functions, and the Adam optimizer with gradient clipping and cosine annealing biases optimization toward low-complexity solutions~\cite{neyshabur2017implicit}. Since the loss minimum is strictly negative, early stopping monitors loss \emph{flattening} rather than vanishing: gradient errors are evaluated every 5 epochs, and training stops after 100 epochs without improvement (patience = 20 evaluations), restoring the best model.
All experiments use random seed 42.

%% ============================================================
%%  Appendix C: Detailed Numerical Results
%% ============================================================
\subsection{Detailed numerical results}\label{app:detailed_results}

%% ── Tables moved from main text (§4.2 and §4.3) ──────────────

\begin{table}[ht!]
\centering
\caption{$M$-scaling of self-test LSE: relative $L^2(\rho)$ errors (\%) for the reference model ($d{=}2$, $N{=}10$, 10 trials per cell).
\textbf{(a)-(b)}: Riemann sum and Trapezoidal quadrature for time integration. The observation time step $\Delta t \in \{10^{-3}, 10^{-2}, 10^{-1}\} $, and the data is generated with time step fixed at $\delta t = 10^{-4}$. 
\textbf{(c)}: Zero-gap regime ($\delta t = \Delta t$). 
The slope is estimated by log-log regression of error vs. $M$. Those slopes matching the theoretical $O(M^{-1/2})$ rate are highlighted in \best{green}.
See Figure~\ref{fig:M_scaling} for the corresponding plots.
}   
\label{tab:M_scaling}
\vspace{-2mm}
\footnotesize
\textbf{(a) Riemann sum ($\delta t{=}10^{-4}$ fixed):}  \\
\setlength{\tabcolsep}{3pt}
\begin{tabular}{r cc cc cc}
\toprule
\multicolumn{7}{c}{\scriptsize $d{=}2$,\; $N{=}10$,\; $T{=}1$,\; $\sigma{=}1$,\; oracle basis,\; $\delta t{=}10^{-4}$} \\
\midrule
& \multicolumn{2}{c}{$\Delta t{=}10^{-3}$} & \multicolumn{2}{c}{$10^{-2}$} & \multicolumn{2}{c}{$10^{-1}$} \\
\cmidrule(lr){2-3}\cmidrule(lr){4-5}\cmidrule(lr){6-7}
$M$ & $\nabla V$ & $\nabla\Phi$ & $\nabla V$ & $\nabla\Phi$ & $\nabla V$ & $\nabla\Phi$ \\
\midrule
20   & $6.4{\pm}2.9$ & $3.1{\pm}1.0$ & $5.4{\pm}3.1$ & $3.6{\pm}1.9$ & $12.2{\pm}7.2$ & $6.0{\pm}3.5$ \\
50   & $4.2{\pm}1.9$ & $1.8{\pm}0.8$ & $4.9{\pm}1.3$ & $2.2{\pm}0.8$ & $5.8{\pm}2.8$ & $4.1{\pm}2.8$ \\
200  & $2.2{\pm}1.2$ & $1.3{\pm}0.6$ & $1.9{\pm}1.0$ & $1.5{\pm}0.6$ & $5.3{\pm}1.8$ & $2.6{\pm}0.9$ \\
1000 & $0.8{\pm}0.5$ & $0.9{\pm}0.2$ & $1.0{\pm}0.5$ & $0.8{\pm}0.3$ & $3.9{\pm}1.1$ & $1.0{\pm}0.7$ \\
2000 & $0.8{\pm}0.3$ & $0.9{\pm}0.2$ & $0.8{\pm}0.4$ & $0.8{\pm}0.3$ & $3.7{\pm}1.0$ & $0.8{\pm}0.4$ \\
5000 & $0.3{\pm}0.2$ & $0.7{\pm}0.1$ & $0.5{\pm}0.3$ & $0.8{\pm}0.2$ & $3.8{\pm}0.4$ & $0.6{\pm}0.3$ \\
\midrule
slope & \best{$-0.5$} & $-0.2$ & \best{$-0.4$} & \best{$-0.3$} & $-0.2$ & \best{$-0.4$} \\
\bottomrule
\end{tabular}

\vspace{4mm}
\textbf{(b) Trapezoid quadrature ($\delta t{=}10^{-4}$ fixed):}
\vspace{1mm}

\setlength{\tabcolsep}{3pt}
\begin{tabular}{r cc cc cc cc}
\toprule
\multicolumn{9}{c}{\scriptsize $d{=}2$,\; $N{=}10$,\; $T{=}1$,\; $\sigma{=}1$,\; oracle basis,\; $\delta t{=}10^{-4}$} \\
\midrule
& \multicolumn{2}{c}{$\Delta t{=}10^{-4}$} & \multicolumn{2}{c}{$10^{-3}$} & \multicolumn{2}{c}{$10^{-2}$} & \multicolumn{2}{c}{$10^{-1}$} \\
\cmidrule(lr){2-3}\cmidrule(lr){4-5}\cmidrule(lr){6-7}\cmidrule(lr){8-9}
$M$ & $\nabla V$ & $\nabla\Phi$ & $\nabla V$ & $\nabla\Phi$ & $\nabla V$ & $\nabla\Phi$ & $\nabla V$ & $\nabla\Phi$ \\
\midrule
20   & $6.6{\pm}2.2$ & $2.8{\pm}1.5$ & $6.4{\pm}2.9$ & $3.1{\pm}1.0$ & $5.9{\pm}3.2$ & $3.5{\pm}1.9$ & $12.3{\pm}6.7$ & $5.7{\pm}3.4$ \\
50   & $4.1{\pm}1.9$ & $2.2{\pm}1.4$ & $4.2{\pm}1.9$ & $1.7{\pm}0.8$ & $5.1{\pm}1.7$ & $2.2{\pm}0.7$ & $5.5{\pm}2.4$ & $3.8{\pm}2.9$ \\
200  & $2.2{\pm}1.1$ & $1.2{\pm}0.6$ & $2.2{\pm}1.2$ & $1.2{\pm}0.6$ & $2.1{\pm}0.9$ & $1.4{\pm}0.6$ & $4.2{\pm}2.0$ & $2.4{\pm}0.8$ \\
1000 & $1.4{\pm}0.6$ & $0.7{\pm}0.3$ & $0.8{\pm}0.5$ & $0.9{\pm}0.2$ & $1.0{\pm}0.5$ & $0.7{\pm}0.2$ & $2.0{\pm}1.0$ & $1.0{\pm}0.6$ \\
2000 & $1.1{\pm}0.5$ & $0.6{\pm}0.2$ & $0.8{\pm}0.4$ & $0.9{\pm}0.2$ & $0.7{\pm}0.3$ & $0.6{\pm}0.3$ & $1.3{\pm}0.6$ & $0.7{\pm}0.3$ \\
5000 & $0.6{\pm}0.4$ & $0.6{\pm}0.1$ & $0.4{\pm}0.2$ & $0.7{\pm}0.1$ & $0.5{\pm}0.2$ & $0.7{\pm}0.2$ & $0.8{\pm}0.4$ & $0.5{\pm}0.3$ \\
\midrule
slope & \best{$-0.4$} & \best{$-0.3$} & \best{$-0.5$} & $-0.2$ & \best{$-0.5$} & \best{$-0.3$} & \best{$-0.5$} & \best{$-0.4$} \\
\bottomrule
\end{tabular}

\vspace{4mm}
\textbf{(c) Zero-gap ($\delta t{=}\Delta t$):} 

\vspace{1mm}
\setlength{\tabcolsep}{3pt}
\begin{tabular}{r cc cc cc cc}
\toprule
\multicolumn{9}{c}{\scriptsize $d{=}2$,\; $N{=}10$,\; $T{=}1$,\; $\sigma{=}1$,\; oracle basis} \\
\midrule
& \multicolumn{2}{c}{$\delta t= \Delta t = 10^{-4}$} & \multicolumn{2}{c}{$\delta t= \Delta t = 10^{-3}$} & \multicolumn{2}{c}{$\delta t= \Delta t = 10^{-2}$} & \multicolumn{2}{c}{$\delta t= \Delta t = 10^{-1}$} \\
\cmidrule(lr){2-3}\cmidrule(lr){4-5}\cmidrule(lr){6-7}\cmidrule(lr){8-9}
$M$ & $\nabla V$ & $\nabla\Phi$ & $\nabla V$ & $\nabla\Phi$ & $\nabla V$ & $\nabla\Phi$ & $\nabla V$ & $\nabla\Phi$ \\
\midrule
20   & $18.0{\pm}8.7$  & $8.3{\pm}5.7$  & $17.2{\pm}8.3$   & $10.1{\pm}4.4$  & $15.6{\pm}6.4$  & $47.3{\pm}3.0$    & $20.0{\pm}2.0$  & $100.8{\pm}1.7$  \\
50   & $11.8{\pm}6.1$  & $7.6{\pm}1.5$  & $12.8{\pm}4.7$   & $7.1{\pm}3.5$   & $15.5{\pm}3.4$  & $47.0{\pm}1.6$    & $20.2{\pm}0.8$  & $100.5{\pm}1.5$  \\
200  & $4.7{\pm}3.6$   & $3.6{\pm}2.0$  & $5.9{\pm}3.3$    & $7.2{\pm}3.2$   & $14.1{\pm}1.4$  & $46.6{\pm}1.2$    & $20.0{\pm}0.5$  & $100.2{\pm}0.5$  \\
1000 & $2.2{\pm}1.3$   & $1.6{\pm}0.5$  & $2.5{\pm}1.1$    & $6.1{\pm}1.3$   & $13.7{\pm}0.6$  & $46.8{\pm}0.3$    & $19.8{\pm}0.2$  & $100.4{\pm}0.3$  \\
2000 & $1.3{\pm}0.8$   & $1.2{\pm}0.3$  & $1.6{\pm}0.8$    & $6.1{\pm}0.9$   & $13.6{\pm}0.5$  & $46.7{\pm}0.3$    & $19.9{\pm}0.1$  & $100.3{\pm}0.2$  \\
5000 & \best{$1.2{\pm}0.5$} & \best{$0.8{\pm}0.4$} & \best{$1.2{\pm}0.5$} & $6.2{\pm}0.4$ & $13.7{\pm}0.3$ & $46.7{\pm}0.2$ & $19.8{\pm}0.1$ & $100.4{\pm}0.1$ \\
\midrule
slope & \best{$-0.5$} & \best{$-0.5$} & \best{$-0.5$} & $-0.1$ & $0.0$ & $0.0$ & $0.0$ & $0.0$ \\
\bottomrule
\end{tabular}
\end{table}

This section reports the additional numerical results referenced in the main text.
We highlight the main findings from each table below.

% \smallskip\noindent
{Table~\ref{tab:M_scaling} reports the $M$-scaling under two quadrature rules.} 
Panels (a) and (b) reports estimation errors for the Riemann-sum and trapezoidal time integration using the observation time step size $\Delta t \in \{10^{-4}, 10^{-3}, 10^{-2}, 10^{-1}\}$,  while the ground-truth data is generated at a fixed fine step $\delta t = 10^{-4}$. The bias floor for these two integration methods are $O(\Delta t)$ for Riemann sum and $O((\Delta t)^2)$ for trapezoidal rule.  In Panel (a), the $O(\Delta t)$ bias dominates at $\Delta t = 10^{-1}$, producing a flat error floor (slope~$=0$) for $\nabla V$ and $\nabla\Phi$. In contrast, Panel (b) shows that the $O((\Delta t)^2)$ bias is smaller than the sampling error $O(M^{-1/2})$, so the error decay rate can be observed. 

Panel~(c) shows that for the discrete-time model (by setting $\delta t = \Delta t$), the $O(\Delta t)$ bias dominates at $\Delta t \ge 10^{-2}$, producing a flat error floor (slope~$=0$) for error decay in $M$. This bias floor is due to the use of Taylor expansion in the self-test approach for the discrete-time model with unlabeled data, and it is different from the numerical error in the continuous-time model. 

% The error floor is higher for $\nabla\Phi$ than $\nabla V$, which is consistent with the fact that the $\Phi$-block of the normal matrix $\mathbf{A}_*$ has smaller minimum eigenvalue than the $V$-block (Table~\ref{tab:cond_all}), making $\nabla\Phi$ more sensitive to the bias. 

Table~\ref{tab:cond_number} reports the condition numbers of the normal matrix $\mathbf{A}_*$ and its diagonal blocks $\mathbf{A}_{VV}$ and $\mathbf{A}_{\Phi\Phi}$ for the self-test LSE of the reference model. 

%% ============================================================
%%  Condition number table 
%% ============================================================
\begin{table}[ht]
\centering
\caption{Condition numbers of the normal matrix and its diagonal blocks in self-test LSE for the reference model.
Rate in $N$ is estimated via log-log regression $\partial\log\kappa / \partial\log N$; Proposition~\ref{prop:joint_coercivity} predicts that $\kappa_* = O(N)$, i.e.\ a rate ~$\,1$.}
\label{tab:cond_number} % \renewcommand{\arraystretch}{0.85}
\setlength{\tabcolsep}{3.5pt}
\begin{tabular}{r rrr rrr rrr}
\toprule
\multicolumn{10}{c}{\scriptsize $\Delta t{=}10^{-2}$,\; $M{=}2000$,\; $T{=}1$,\; $\sigma{=}1$,\; oracle basis} \\
\midrule
& \multicolumn{3}{c}{$d=2$}
& \multicolumn{3}{c}{$d=5$}
& \multicolumn{3}{c}{$d=10$} \\
\cmidrule(lr){2-4} \cmidrule(lr){5-7} \cmidrule(lr){8-10}
$N$
  & $\kappa_*$ & $\kappa_{VV}$ & $\kappa_{\Phi\Phi}$
  & $\kappa_*$ & $\kappa_{VV}$ & $\kappa_{\Phi\Phi}$
  & $\kappa_*$ & $\kappa_{VV}$ & $\kappa_{\Phi\Phi}$ \\
\midrule
5   & $37$  & $19.7$ & $12.8$ & $42$   & $35.4$ & $6.2$  & $170$  & $156$  & $3.0$ \\
10  & $53$  & $20.2$ & $6.8$  & $44$   & $32.4$ & $4.7$  & $232$  & $166$  & $6.9$ \\
20  & $80$  & $20.5$ & $3.7$  & $50$   & $31.4$ & $4.5$  & $533$  & $171$  & $14$ \\
50  & $123$ & $20.8$ & $1.9$  & $69$   & $30.7$ & $6.1$  & $1629$ & $177$  & $31$ \\
100 & $155$ & $20.9$ & $1.5$  & $98$   & $30.4$ & $9.1$  & $3426$ & $179$  & $50$ \\
\midrule
Rate in $N$
  & \best{$0.5$} & $0.0$ & $-0.7$
  & \best{$0.3$} & $0.0$ & $0.1$
  & \best{$1.1$} & $0.0$ & $0.9$ \\
$\lambda_{\max}$
  & \multicolumn{3}{c}{$\approx 0.23$}
  & \multicolumn{3}{c}{$\approx 0.38$}
  & \multicolumn{3}{c}{$\approx 0.92$} \\
\bottomrule
\end{tabular}
\end{table}

Figure~\ref{fig:all_methods} shows the estimators from all methods for the four stress test radial models when $d=2$. In particular, it shows that the RBF basis, when applied with MLE for the labeled data (either ground truth or recovered by Sinkhorn) or with self-test LSE, produces oscillatory estimators with errors larger than those of the other methods, confirming the importance of basis specification or finetuning for least squares estimation. 
\begin{figure}[htb!]
\centering
\includegraphics[width=\textwidth]{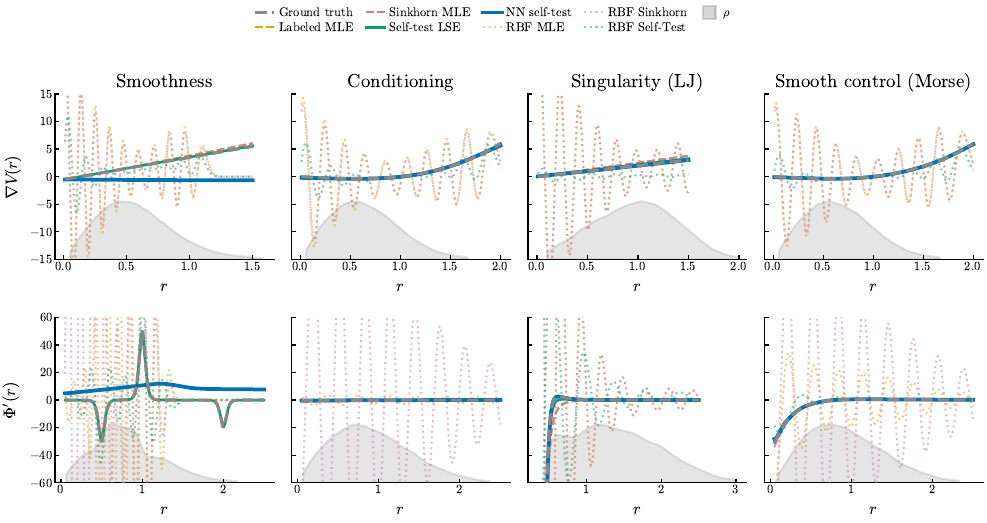}
\caption{Typical estimators of $V'$ and $\Phi'$ for the radial potentials when $d=2$. Top left: Smoothness test. Top right: Conditioning test. Bottom left: Singularity test (LJ). Bottom right: Smooth control (Morse). Gray shading indicates the empirical density of observed distances.}
\label{fig:all_methods}%
\label{fig:nn_individual_a}\label{fig:nn_individual_b}%
\label{fig:nn_individual_lj}\label{fig:nn_individual_morse}%
\end{figure}

%%%%%%%%%%%%%%%%%

\section{Proofs for the error bound of the parametric estimator}\label{app:proofs_parametric}

%%%%%%%%%%%
\subsection{Continuum expectation limit and coercivity}

\begin{lemma}[Continuum expectation normal equation]\label{lemma:population_limit} Under Assumption~\ref{ass:parametric_theory}, denote 
\begin{equation}\label{eq:population_Ab}
    \mathbf A_* = \frac1T\int_0^T \mathbb E\!\left[\frac1N\sum_{i=1}^N \mathbf F_{i,t}^\top \mathbf F_{i,t}\right]dt,
    \qquad
    \mathbf b_* = \frac{\sigma^2}{2T}\int_0^T \mathbb E[\boldsymbol\delta_t]dt - \frac1T\mathbb E[\boldsymbol h_T-\boldsymbol h_0]. 
\end{equation}
Then, the true parameter $\theta_*$ satisfies
\begin{equation}\label{eq:population_normal_skeleton}
    \mathbf A_*\theta_* = \mathbf b_*,
\end{equation}
\end{lemma}

\begin{proof}
    Recall the definition of the vector-valued regression functions for the forces $\mathbf{F}_{i,t}= \mathbf{F}_i(\bX_t)$, diffusion correction $\boldsymbol\delta_t = \boldsymbol\delta(\bX_t)$, and potential energy $\mathbf{h}_t = \mathbf{h}(\bX_t)$ in \eqref{eq:regressionVectors}.      
    Note that the drift parametrized by $\theta$ at particle $i$ at time $t$ is simply
\begin{equation}\label{eq:drift_F}
    \nabla V_\alpha(X_t^i) + \frac1N\sum_{j\ne i}\nabla \Phi_\beta(X_t^i-X_t^j)
    = \mathbf F_{i,t}\theta,
\end{equation}
and direct differentiation of $\mathbf{h}(\bX)$ gives
\begin{equation*}
    \nabla_{\bX}\boldsymbol h_t = \frac1N(\mathbf F_{1,t},\dots,\mathbf F_{N,t}),
    \qquad
    \Delta_{\bX}\boldsymbol h_t = \boldsymbol\delta_t. 
\end{equation*}

Then, with the true parameter $\theta_*$, the particle system has the form
\[
    dX_t^i = -\mathbf F_{i,t}\theta_*\,dt + \sigma\,dW_t^i,
    \qquad i=1,\dots,N.
\]
Applying It\^o's formula componentwise to $\boldsymbol h_t$ gives
\begin{equation}\label{eq:ito_h}
    d\boldsymbol h_t
    = -\frac1N\sum_{i=1}^N \mathbf F_{i,t}^\top\mathbf F_{i,t}\,\theta_*\,dt
    + \frac{\sigma^2}{2}\boldsymbol\delta_t\,dt
    + \frac{\sigma}{N}\sum_{i=1}^N \mathbf F_{i,t}^\top\,dW_t^i.
\end{equation}
Integrating over $[0,T]$, taking expectations, and using the fact that the stochastic integral has mean zero, we obtain
\[
    \frac1T\int_0^T \mathbb E\!\left[\frac1N\sum_{i=1}^N \mathbf F_{i,t}^\top\mathbf F_{i,t}\right]dt\,\theta_*
    = \frac{\sigma^2}{2T}\int_0^T \mathbb E[\boldsymbol\delta_t]dt - \frac1T\mathbb E[\boldsymbol h_T-\boldsymbol h_0],
\]
which is exactly \eqref{eq:population_normal_skeleton}.
\end{proof}

\begin{lemma}[Interaction coercivity]\label{lem:interaction_only}
Under Assumption~\ref{ass:parametric_theory}{\rm (1)--(2)}, for every $\beta\in\R^{K_\Phi}$,
\begin{equation}\label{eq:interaction_only_skeleton}
    \frac1T\int_0^T \mathbb E\big[\frac1 N \sum_{i=1}^N \big|\frac1N\sum_{j\ne i}\nabla\Phi_\beta(X_t^i-X_t^j)\big|^2\big]dt
    \ge \frac{c_N}{T }\int_0^T \mathbb E\left[ |\nabla\Phi_\beta(X_t^1-X_t^2)|^2 \right] dt
    = c_N \beta^\top \mathbf G_\Phi \beta, 
\end{equation}
where $c_N = \frac{N-1}{N^2}$ and  $\mathbf G_\Phi$ is defined in \eqref{eq:GV_GPhi_skeleton}.
\end{lemma}

\begin{proof}
Assumption~\ref{ass:parametric_theory}(1) ensures that the second moment of the interaction potential exists and is integrable in time. Expanding the square and applying exchangeability, we have
\[
\begin{aligned}
& \mathbb E\left[\frac{1}{N}\sum_{i=1}^N\big|\frac1N\sum_{j\ne i}\nabla\Phi_\beta(X_t^i-X_t^j)\big|^2\right] \\
= &  \frac1{N^3}\sum_{i=1}^N\sum_{j\ne i} \mathbb E\big[|\nabla\Phi_\beta(X_t^i-X_t^j)|^2\big]  + \frac1{N^3}\sum_{i=1}^N\sum_{\substack{j,k\ne i \\ j\ne k}} \mathbb E\big[\nabla\Phi_\beta(X_t^i-X_t^j)\cdot \nabla\Phi_\beta(X_t^i-X_t^k)\big] \\
 = & \frac{N-1}{N^2}\mathbb E[|\nabla\Phi_\beta(X_t^1-X_t^2)|^2] + \frac{(N-1)(N-2)}{N^2} \mathbb E[\nabla\Phi_\beta(X_t^1-X_t^2)\cdot \nabla\Phi_\beta(X_t^1-X_t^3)] 
\end{aligned}
\]
For the off-diagonal part, conditional independence and exchangeability yield
\[
\mathbb E\big[\nabla\Phi_\beta(X_t^1-X_t^2)\cdot \nabla\Phi_\beta(X_t^1-X_t^3)\big]
= \mathbb E\Big[\big|\mathbb E[\nabla\Phi_\beta(X_t^1-X_t^2)\mid \mathcal F_t]\big|^2\Big]\ge 0.
\]
Therefore,
\[
\mathbb E\big[ \frac1N\sum_{i=1}^N\big|\frac1N\sum_{j\ne i}\nabla\Phi_\beta(X_t^i-X_t^j)\big|^2\big]
\ge \frac{N-1}{N^2}\mathbb E\big[|\nabla\Phi_\beta(X_t^1-X_t^2)|^2\big].
\]
Integrating over $t\in[0,T]$ gives \eqref{eq:interaction_only_skeleton}.
\end{proof}

\begin{proposition}[Coercivity for joint estimation]\label{prop:joint_coercivity}
Under Assumption~\ref{ass:parametric_theory}, for every $\theta=(\alpha,\beta)$, 
\begin{equation}\label{eq:joint_coercivity}
    \theta^\top \mathbf A_*\theta  = \frac1T\int_0^T \mathbb E\big[J_{diss}^\theta(\bX_t)\big]dt
    \ge
    (1-c_0)\left(\alpha^\top \mathbf G_V \alpha + c_N \beta^\top \mathbf G_\Phi \beta\right).
\end{equation}
In particular, $\mathbf A_*$ is positive definite with $\lambda_{min}(\mathbf A_*) \ge (1-c_0)\min\{\lambda_{min}(\mathbf G_V), c_N \lambda_{min}(\mathbf G_\Phi)\}$.
\end{proposition}

\begin{proof}
Expand $J_{diss}^\theta$ and apply the exchangeability, we have 
\begin{align*}
 & \mathbb{E}[J_{diss}^\theta(\bX_t)] =   \frac{1}{N} \sum_{i = 1}^N  \mathbb{E} \big|\nabla V_\alpha(X_{t}^i)+ \frac{1}{N}\sum_{j = 1}^N \nabla \Phi_\beta (X_{t}^i- X_{t}^j) \big|^2 \\
 % = &  \frac{1}{N} \sum_{i = 1}^N \mathbb{E} \big|\nabla V_\alpha(X_{t}^i)\big|^2 + \frac{1}{N} \sum_{i = 1}^N \mathbb{E} \big| \frac 1 N \sum_{j = 1}^N  \nabla \Phi_\beta (X_{t}^i- X_{t}^j)\big|^2 + \frac{2}{N^2} \sum_{i = 1}^N \sum_{j = 1, j\neq i}^N \mathbb{E} \big[\nabla V_\alpha(X_{t}^i) \nabla \Phi_\beta (X_{t}^i- X_{t}^j)\big]\\
 = & \mathbb{E} \big|\nabla V_\alpha(X_{t}^1)\big|^2 + \frac{1}{N} \sum_{i = 1}^N \mathbb{E} \big| \frac 1 N \sum_{j = 1}^N  \nabla \Phi_\beta (X_{t}^i- X_{t}^j)\big|^2 +  \frac{2(N-1)}{N}\mathbb{E} \big[\nabla V_\alpha(X_{t}^1) \cdot \nabla \Phi_\beta (X_{t}^1- X_{t}^2)\big].
\end{align*}
Integrating over $t\in[0,T]$ and applying the bound in Lemma~\ref{lem:interaction_only} for the pure interaction term and the bound for the mixed term in Assumption~\ref{ass:parametric_theory}(3), we have
\[
\begin{aligned}
\frac1T\int_0^T \mathbb E\big[J_{diss}^\theta(\bX_t)\big]dt
 \ge  & (1-c_0)\left(\int_0^T \mathbb E\big[|\nabla V_\alpha(X_t^1)|^2\big]dt + c_N \int_0^T \mathbb E\big[|\nabla \Phi_\beta(X_t^1-X_t^2)|^2\big]dt\right) \\
 = &  (1-c_0)\left(\alpha^\top \mathbf G_V \alpha + c_N \beta^\top \mathbf G_\Phi \beta\right). 
\end{aligned}
\]
Together with the fact that $ \theta^\top \mathbf A_*\theta  = \frac1T\int_0^T \mathbb E\big[J_{diss}^\theta(\bX_t)\big]dt$, this gives \eqref{eq:joint_coercivity}.
\end{proof}

\begin{remark}[Negative minimum]
    The joint coercivity in Proposition {\rm \ref{prop:joint_coercivity}} implies that $\theta_*$ is the unique minimizer of the continuum expectation loss function $\bar{\mathcal L}_*$ defined by  
\[
    \bar{\mathcal L}_*(\theta) := \frac12 \theta^\top \mathbf A_*\theta - (\mathbf b_*)^\top\theta, 
\]
and the minimum value is $\bar{\mathcal L}_*(\theta_*) = -\frac12 (\theta_*)^\top \mathbf A_*\theta_* = -\frac12 (\mathbf b_*)^\top (\mathbf A_*)^{-1}\mathbf b_* < 0$ when $\theta_*\neq 0$.   
Thus, when $M$ is sufficiently large and $\Delta t$ is sufficiently small, the empirical loss $\mathcal L$ is close to $\bar{\mathcal L}_*$, and has a strictly negative minimum. Furthermore, the quadratic form of $\bar{\mathcal L}_*$ around $\theta_*$ is given by
\begin{equation*}
    \bar{\mathcal L}_*(\theta) - \bar{\mathcal L}_*(\theta_*)
    = \frac12 (\theta-\theta_*)^\top \mathbf A_*(\theta-\theta_*)
    \ge \frac12 \lambda_{\min}(\mathbf A_*) |\theta-\theta_*|^2, 
\end{equation*}
so the joint coercivity implies that the loss function is convex at the minimum $\theta_*$. This is a key property that enables one to establish error bounds for the estimator $\widehat\theta_{M,\Delta t}$ by quantifying the perturbation of the empirical loss relative to its continuum-expectation counterpart. 
\end{remark}

\subsection{Fixed-grid concentration and perturbation}
For a fixed observation grid, define
\begin{equation}\label{eq:EAb}
    \mathbf A_\Delta := \mathbb E[\mathbf A_{M,\Delta t}],
    \qquad
    \mathbf b_\Delta := \mathbb E[\mathbf b_{M,\Delta t}],
    \qquad
    \theta_\Delta := \mathbf A_\Delta^{-1}\mathbf b_\Delta.
\end{equation}

\begin{lemma}[Discretization error]\label{lem:discretization_bound}
Under Assumption~\ref{ass:parametric_theory}, $\mathbf A_\Delta$ and $\mathbf b_\Delta$ in \eqref{eq:EAb} satisfy
\begin{equation}\label{eq:discretization_input_skeleton}
    \|\mathbf A_\Delta-\mathbf A_*\| \leq L_A \Delta t, \quad |\mathbf b_\Delta-\mathbf b_*| \le  \frac{\sigma^2 L_b}{2} \Delta t, 
\end{equation}
where $L_A = C_A N\Big(2d c_{max,1}+\frac{\sigma^2}{2}\Big)c_{max,3}^2$ and $L_b = C_b N\Big(2d c_{max,1}+\frac{\sigma^2}{2}\Big)c_{max,4}$ with constants $C_A$ and $C_b$ depending only on $d$ and $K$. 
\end{lemma}

\begin{proof}
Set $ \mathbf G(t) := \mathbb E\left[\frac1N\sum_{i=1}^N \mathbf F_{i,t}^\top\mathbf F_{i,t}\right]$ and $\mathbf d(t) := \mathbb E[\boldsymbol\delta_t]$.  By definition, 
$$
\begin{aligned}
\mathbf A_\Delta & = \frac1L\sum_{\ell=0}^{L-1} \mathbf G(t_\ell),  \quad\, \,   \mathbf b_\Delta  = \frac{\sigma^2}{2L}\sum_{\ell=0}^{L-1} \mathbf d(t_\ell) - \frac{\mathbb E[\boldsymbol h_T-\boldsymbol h_0]}{T} , \\
 \mathbf A_* & =  \frac1T\int_0^T \mathbf G(t)dt, \quad  
 \, \mathbf b_* = \frac{\sigma^2}{2T} \int_0^T \mathbf d(t)dt - \frac{\mathbb E[\boldsymbol h_T-\boldsymbol h_0]}{T}.
\end{aligned}
$$
Since $\Delta t=T/L$, we have 
\[
    \mathbf A_\Delta-\mathbf A_*
    = \frac1T\sum_{\ell=0}^{L-1} \int_{t_\ell}^{t_{\ell+1}} \big(\mathbf G(t_\ell)-\mathbf G(t)\big)dt, \quad \mathbf b_\Delta-\mathbf b_* = \frac{\sigma^2}{2T}\sum_{\ell=0}^{L-1} \int_{t_\ell}^{t_{\ell+1}} \big(\mathbf d(t_\ell)-\mathbf d(t)\big)dt.
\]
If $\mathbf G$ and $\mathbf d$ are Lipschitz continuous with constants $L_A$ and $L_b$, then
\[
    \|\mathbf A_\Delta-\mathbf A_*\|
    \le \frac1T\sum_{\ell=0}^{L-1} \int_{t_\ell}^{t_{\ell+1}} L_A |t-t_\ell|dt
    \le L_A \Delta t, 
    \quad |\mathbf b_\Delta-\mathbf b_*| \le \frac{\sigma^2 L_b}{2}\Delta t.
\]

Next, we show that $\mathbf G$ and $\mathbf d$ are indeed Lipschitz continuous by using the generator of the diffusion process, which requires the $C^4_b$ regularity of the basis functions. 
Denote \[
\Gamma_{kk'}(\bX):=\frac1N\sum_{i=1}^N\big(\mathbf F_i(\bX)^\top \mathbf F_i(\bX)\big)_{kk'},
\quad
D_k(\bX):=\boldsymbol\delta(\bX)_k, \quad k,k'=1,\dots,K.
\]
Then $\mathbf G(t)_{kk'}=\mathbb E[\Gamma_{kk'}(\bX_t)]$ and $\mathbf d(t)_k=\mathbb E[D_k(\bX_t)]$. 

By the It\^o formula for $\bX_t$, for every $s,t\in[0,T]$,
\begin{align*}
\mathbf G(t)_{kk'}-\mathbf G(s)_{kk'} &= \mathbb E[\Gamma_{kk'}(\bX_t)]-\mathbb E[\Gamma_{kk'}(\bX_s)] = \int_s^t \mathbb E[(\mathcal L_\star \Gamma_{kk'})(\bX_r)]dr, \\
\mathbf d(t)_k-\mathbf d(s)_k &= \int_s^t \mathbb E[(\mathcal L_\star D_k)(\bX_r)]dr,
\end{align*}
where $\mathcal L_\star$ is the generator of the true particle system, i.e., for any $H\in C_b^2((\R^d)^N)$,
\[
(\mathcal L_\star H)(\bX) 
=  -\sum_{i=1}^N b_i^\star(\bX)\cdot \nabla_{x_i}H(\bX) +\frac{\sigma^2}{2}\sum_{i=1}^N \Delta_{x_i}H(\bX),
\]
where $b_i^\star(\bX) = \nabla V_\star(X^i)+\frac1N\sum_{j\ne i}\nabla\Phi_\star(X^i-X^j)$. 
Since the true potentials are in the span of the basis functions, the true drift is uniformly bounded, $
B_\star:=\sup_{\bX\in(\R^d)^N}\max_{1\le i\le N}|b_i^\star(\bX)|< 2d c_{max,1}.
$
Then, for every $H\in C_b^2((\R^d)^N)$,
\[
 \Big|\mathbb E[(\mathcal L_\star H)(\bX_t)]\Big|
\le  
N B_\star \|\nabla H\|_\infty
+\frac{\sigma^2}{2}N \|\nabla^2 H\|_\infty . 
\]
Applying to $\Gamma_{kk'}$ and $D_k$, we have
\[
\begin{aligned}
|\mathbf G(t)_{kk'}-\mathbf G(s)_{kk'}| & \le \int_s^t \Big|\mathbb E[(\mathcal L_\star \Gamma_{kk'})(\bX_r)]\Big|dr
\le N\Big(B_\star+\frac{\sigma^2}{2}\Big)\|\Gamma_{kk'}\|_{C_b^2((\R^d)^N)} |t-s|, \\
|\mathbf d(t)_k-\mathbf d(s)_k| & \le \int_s^t \Big|\mathbb E[(\mathcal L_\star D_k)(\bX_r)]\Big|dr
\le N\Big(B_\star+\frac{\sigma^2}{2}\Big)\|D_k\|_{C_b^2((\R^d)^N)} |t-s|.
\end{aligned}
\]
By the definition of $\Gamma_{kk'}$ and $D_k$, each entry of $\mathbf G(t)$ and $\mathbf d(t)$ is a spatial average of the corresponding $\Gamma_{kk'}$ and $D_k$. Thus, the Lipschitz constants for $\mathbf G$ and $\mathbf d$ are controlled by the $C_b^2$ norms of $\Gamma_{kk'}$ and $D_k$, which in turn are controlled by the $C_b^4$ norms of the basis functions: 
\[
\max_{k,k'}\|\Gamma_{kk'}\|_{C_b^2((\R^d)^N)}
\le C_A\,c_{max,3}^2,
\qquad
\max_{1\le k\le K}\|D_k\|_{C_b^2((\R^d)^N)}
\le C_b\,c_{max,4},   
\]
for constants $C_A=C_A(d,K)$ and $C_b=C_b(d,K)$ independent of $t$. Consequently, $\mathbf G$ and $\mathbf d$ are Lipschitz continuous with constants $L_A=C_A N\Big(2d c_{max,1}+\frac{\sigma^2}{2}\Big)c_{max,3}^2$ and $L_b=C_b N\Big(2d c_{max,1}+\frac{\sigma^2}{2}\Big)c_{max,4}$, respectively.
\end{proof}

\begin{lemma}[Trapezoidal discretization error]\label{lem:trap_discretization}
Under Assumption~\ref{ass:parametric_theory} with the strengthened regularity $\psi_k,\phi_\ell^{sym}\in C_b^6(\R^d)$, define
\[
    \mathbf A_\Delta^{\mathrm{trap}} := \frac{1}{2L}\sum_{\ell=0}^{L-1}\big[\mathbf G(t_\ell)+\mathbf G(t_{\ell+1})\big],
    \qquad
    \mathbf b_\Delta^{\mathrm{trap}} := \frac{\sigma^2}{4L}\sum_{\ell=0}^{L-1}\big[\mathbf d(t_\ell)+\mathbf d(t_{\ell+1})\big] - \frac{\mathbb E[\boldsymbol h_T-\boldsymbol h_0]}{T}.
\]
Then,
\begin{equation}\label{eq:trap_discretization}
    \|\mathbf A_\Delta^{\mathrm{trap}}-\mathbf A_*\| \leq L_A' (\Delta t)^2, \qquad |\mathbf b_\Delta^{\mathrm{trap}}-\mathbf b_*| \le  \frac{\sigma^2 L_b'}{2} (\Delta t)^2,
\end{equation}
where $L_A' = \frac{1}{12}C_A'\,N^2\Big(2d\,c_{max,1}+\frac{\sigma^2}{2}\Big)^2c_{max,5}^2$ and $L_b' = \frac{1}{12}C_b'\,N^2\Big(2d\,c_{max,1}+\frac{\sigma^2}{2}\Big)^2c_{max,6}$ with constants $C_A',C_b'$ depending only on $d$ and $K$.
\end{lemma}

\begin{proof}
For any $f\in C^2([0,T])$, the composite trapezoidal rule with step $\Delta t = T/L$ satisfies
\[
    \bigg|\frac1T\int_0^T f(t)\,dt - \frac{1}{2L}\sum_{\ell=0}^{L-1}\big[f(t_\ell)+f(t_{\ell+1})\big]\bigg|
    \le \frac{(\Delta t)^2}{12}\sup_{t\in[0,T]}|f''(t)|.
\]
Applying this entrywise to $\mathbf G(t)$ and $\mathbf d(t)$, it suffices to bound $\sup_t|\mathbf G''(t)_{kk'}|$ and $\sup_t|\mathbf d''(t)_k|$.

By the proof of Lemma~\ref{lem:discretization_bound}, $\mathbf G'(t)_{kk'} = \mathbb E[(\mathcal L_\star\Gamma_{kk'})(\bX_t)]$. Differentiating once more,
\[
    \mathbf G''(t)_{kk'} = \mathbb E[(\mathcal L_\star^2\Gamma_{kk'})(\bX_t)].
\]
Since $\Gamma_{kk'}$ depends on first derivatives of the basis functions, $\Gamma_{kk'}\in C_b^4((\R^d)^N)$ when the basis functions are in $C_b^5(\R^d)$, and $\mathcal L_\star^2\Gamma_{kk'}$ is bounded when the basis functions are in $C_b^6(\R^d)$. By the generator bound in the proof of Lemma~\ref{lem:discretization_bound} applied twice,
\[
    \sup_{t\in[0,T]}|\mathbf G''(t)_{kk'}| \le N^2\Big(B_\star+\frac{\sigma^2}{2}\Big)^2 \|\Gamma_{kk'}\|_{C_b^4((\R^d)^N)}
    \le C_A' N^2\Big(2dc_{max,1}+\frac{\sigma^2}{2}\Big)^2 c_{max,5}^2.
\]
Consequently, $\|\mathbf A_\Delta^{\mathrm{trap}}-\mathbf A_*\| \le L_A'(\Delta t)^2$. The bound for $\mathbf b_\Delta^{\mathrm{trap}}-\mathbf b_*$ follows identically with $D_k$ in place of $\Gamma_{kk'}$ and $c_{max,6}$ in place of $c_{max,5}$.
\end{proof}

\begin{lemma}[Concentration of normal matrix and vector]\label{lem:concentration}
Under Assumption~\ref{ass:parametric_theory}, for every $\eta\in(0,1)$, with probability at least $1-\eta$,
\begin{equation}\label{eq:concentration_matrix_skeleton}
    \|\mathbf A_{M,\Delta t}-\mathbf A_\Delta\|_F \le \sqrt{\frac{2C_A}{M\eta}}, 
    \qquad 
    |\mathbf b_{M,\Delta t}-\mathbf b_\Delta|
    \le \sqrt{\frac{2C_b}{M\eta}}, 
\end{equation}
where $ C_A := 4d^2 K^2 c_{max,2}^4$, $C_b= (\sigma^4+ \frac{16}{T^2}) K^2 c_{max,4}^2$, and $\mathbf A_\Delta$ and $\mathbf b_\Delta$ are defined in \eqref{eq:EAb}. 
\end{lemma}

\begin{remark}[Sharper concentration and nonparametric estimation] 
Sharper concentration bounds can be obtained by replacing the Chebyshev bounds in Lemma~{\rm\ref{lem:concentration}} with Bernstein-type bounds under stronger tail assumptions. This is particularly useful in nonparametric estimation settings, where finer control over the concentration of empirical quantities is needed. The main challenge for nonparametric estimation is the coercivity condition in the infinite-dimensional function space. We leave this as a direction for future work.
\end{remark}

\begin{proof} We use Chebyshev's inequality to prove the concentration. 
Denote 
\[
    \mathbf A
    := \frac1L\sum_{\ell=0}^{L-1}\frac1N\sum_{i=1}^N
    \mathbf F_{i,t_\ell}^\top \mathbf F_{i,t_\ell},
    \qquad
    \mathbf b
    := \frac1L\sum_{\ell=0}^{L-1}\frac{\sigma^2}{2}\boldsymbol\delta_{t_\ell}
    - \frac1T(\boldsymbol h_T-\boldsymbol h_0).
\]
Then, $\mathbf A_\Delta = \E{\mathbf A}$, $\mathbf b_\Delta = \E{\mathbf b}$. With $\mathbf A^{(m)}$ and $\mathbf b^{(m)}$ being i.i.d.\ copies of $\mathbf A$ and $\mathbf b$, we have 
\[
    \mathbf A_{M,\Delta t}-\mathbf A_\Delta
    = \frac1M\sum_{m=1}^M \bigl(\mathbf A^{(m)}-\mathbf A_\Delta\bigr),
    \qquad
    \mathbf b_{M,\Delta t}-\mathbf b_\Delta
    = \frac1M\sum_{m=1}^M \bigl(\mathbf b^{(m)}-\mathbf b_\Delta\bigr). 
\]

For the matrix term, each column of $\mathbf F_i(\bX)$ is either $\nabla\psi_k(X^i)$ or
$\frac1N\sum_{j\ne i}\nabla\phi^{sym}_{k-K_V}(X^i-X^j)$, so Assumption~\ref{ass:parametric_theory}(1) gives $|\mathbf F_i(\bX)_{:,k}| \le c_{max,2}$. Hence, 
\[
    \|\mathbf A\|_F
    \le \frac1L\sum_{\ell=0}^{L-1}\frac1N\sum_{i=1}^N
    \bigl\|\mathbf F_{i,t_\ell}^\top \mathbf F_{i,t_\ell}\bigr\|_F
    \le d K c_{max,2}^2 .
\]
Consequently, with $\|\mathbf A_\Delta\|_F \le \E{ \|\mathbf A\|_F}$, we have 
\[
    \mathbb E\|\mathbf A-\mathbf A_\Delta\|_F^2
    \le 4 \mathbb E\|\mathbf A\|_F^2
    \le 4 d^2 K^2 c_{max,2}^4 = C_A. 
\]
By the independence of the trajectories,
\[
    \mathbb E\|\mathbf A_{M,\Delta t}-\mathbf A_\Delta\|_F^2
    = \frac1M \mathbb E\|\mathbf A-\mathbf A_\Delta\|_F^2
    \le \frac{C_A}{M}.
\]
Then, Chebyshev's inequality yields
$    \mathbb P\!\left(
        \|\mathbf A_{M,\Delta t}-\mathbf A_\Delta\|_F
        > \sqrt{\frac{2C_A}{M\eta}}
    \right)
    % \le  \E{\|\mathbf A_{M,\Delta t}-\mathbf A_\Delta\|_F^2} / (\frac{2C_A}{M\eta})
    \le \frac{\eta}{2}$. 

For the vector term, note that 
\[
 \mathbf b - \mathbf b_\Delta = \frac1L\sum_{\ell=0}^{L-1}\frac{\sigma^2}{2}\bigl(\boldsymbol\delta_{t_\ell}-\mathbb E[\boldsymbol\delta_{t_\ell}]\bigr) - \frac1T\bigl((\boldsymbol h_T-\boldsymbol h_0) - \mathbb E[\boldsymbol h_T-\boldsymbol h_0]\bigr).
\]
By Assumption \ref{ass:parametric_theory}(1), $|\boldsymbol\delta(\bX)_k| \le c_{max,4}$ for every $k$, so $ \|\boldsymbol\delta_{t_\ell}-\mathbb E[\boldsymbol\delta_{t_\ell}]\| \le 2 K c_{max,4}$. By the definition of $\boldsymbol h(\bX)$ and the boundedness of the basis functions, $|\boldsymbol h(\bX)_k| \le c_{max,4}$ for every $k$, so $\|(\boldsymbol h_T-\boldsymbol h_0) - \mathbb E[\boldsymbol h_T-\boldsymbol h_0]\| \le 4K c_{max,4}$. Hence,
\[
\E{\|\mathbf b - \mathbf b_\Delta\|^2} \le \sigma^4 K^2 c_{max,4}^2 + \frac{16K^2}{T^2}  c_{max,4}^2 = (\sigma^4+ \frac{16}{T^2}) K^2 c_{max,4}^2 = C_b.
\]
Then, Chebyshev's inequality gives
$    \mathbb P\!\left(
        |\mathbf b_{M,\Delta t}-\mathbf b_\Delta|
        > \sqrt{\frac{2C_b}{M\eta}}
    \right)
    \le \frac{\eta}{2}$. 
\end{proof}

\subsection{Proof of the main error bound}

\begin{proof}[Proof of Theorem~\ref{thm:error_bound}]
We decompose the error into sampling error and discretization error:
\begin{equation}\label{eq:error_split_theta}
 |\widehat  \theta_{M,\Delta t}-\theta_*|
 \le |\widehat  \theta_{M,\Delta t}-\theta_\Delta|
    + |\theta_\Delta-\theta_*|.
\end{equation}
We first bound the discretization error $|\theta_\Delta-\theta_*|$, and then the sampling error $|\widehat  \theta_{M,\Delta t}-\theta_\Delta|$.

\noindent\textit{Step 1: bound on the discretization error.}
By Proposition~\ref{prop:joint_coercivity}, $\mathbf A_*$ is positive definite, hence $\mu_*:=\lambda_{\min}(\mathbf A_*)>0$. By Lemma~\ref{lem:discretization_bound}, $\|\mathbf A_\Delta-\mathbf A_*\| \le L_A \Delta t$ and $|\mathbf b_\Delta-\mathbf b_*| \le \frac{\sigma^2 L_b}{2}\Delta t$. 
Choose $\Delta_0 = \frac{\mu_*}{4L_A}$ so that $L_A\Delta_0\le \mu_*/4$. Then, for every $\Delta t\le \Delta_0$,
\[
     \lambda_{\min}(\mathbf A_\Delta)
     \ge \lambda_{\min}(\mathbf A_*)-\|\mathbf A_\Delta-\mathbf A_*\|
     \ge \frac{3\mu_*}{4},
\]
so $\mathbf A_\Delta$ is invertible and $\|\mathbf A_\Delta^{-1}\|\le \frac{4}{3\mu_*}$.
Using $\mathbf A_\Delta\theta_\Delta=\mathbf b_\Delta$ and $\mathbf A_*\theta_*=\mathbf b_*$,
\[
     \theta_\Delta-\theta_*
     = \mathbf A_\Delta^{-1}\Big((\mathbf b_\Delta-\mathbf b_*)-(\mathbf A_\Delta-\mathbf A_*)\theta_*\Big),
\]
and therefore
\[
     |\theta_\Delta-\theta_*|
     \le \|\mathbf A_\Delta^{-1}\|
     \Big(
          |\mathbf b_\Delta-\mathbf b_*|
          + \|\mathbf A_\Delta-\mathbf A_*\|\,|\theta_*|
     \Big)
     \le C_1 \Delta t, 
\]
where $C_1:=\frac{4}{3\mu_*}\left(\frac{\sigma^2L_b}{2}+L_A|\theta_*|\right)$.

\smallskip
\noindent\textit{Step 2: bound on the sampling error.}
Let $\Omega_{M,\eta}$ be the event on which the concentration bounds in
\eqref{eq:concentration_matrix_skeleton} in Lemma~\ref{lem:concentration} hold, and we have $ \mathbb P(\Omega_{M,\eta})\ge 1-\eta$.

Choose $M_0 = \frac{32 C_A}{\mu_*^2}$, then for every $M\ge M_0/\eta$, we have $ \sqrt{\frac{2C_A}{M\eta}} \le \frac{\mu_*}{4}$. Then, on $\Omega_{M,\eta}$,
\[
     \|\mathbf A_{M,\Delta t}-\mathbf A_\Delta\|
     \le \|\mathbf A_{M,\Delta t}-\mathbf A_\Delta\|_F
     \le \frac{\mu_*}{4},
\]
and, since $\lambda_{\min}(\mathbf A_\Delta)\ge 3\mu_*/4$, we have 
\[
     \lambda_{\min}(\mathbf A_{M,\Delta t})
     \ge \lambda_{\min}(\mathbf A_\Delta)-\|\mathbf A_{M,\Delta t}-\mathbf A_\Delta\|
     \ge \frac{\mu_*}{2}.
\]
Thus $\mathbf A_{M,\Delta t}$ is invertible and $ \|\mathbf A_{M,\Delta t}^{-1}\|\le \frac{2}{\mu_*}$. 

Using the equations $\mathbf A_{M,\Delta t}\widehat\theta_{M,\Delta t}=\mathbf b_{M,\Delta t}$ and $\mathbf A_\Delta\theta_\Delta=\mathbf b_\Delta$, we have 
\[
     \widehat\theta_{M,\Delta t}-\theta_\Delta
     = \mathbf A_{M,\Delta t}^{-1}
     \Big(
          (\mathbf b_{M,\Delta t}-\mathbf b_\Delta)
          -(\mathbf A_{M,\Delta t}-\mathbf A_\Delta)\theta_\Delta
     \Big).
\]
Also, Step 1 gives $|\theta_\Delta|\leq |\theta_*|+C_1\Delta_0 =: B_\theta$. Hence, on $\Omega_{M,\eta}$,
\[
\begin{aligned}
     |\widehat\theta_{M,\Delta t}-\theta_\Delta|
     &\le \|\mathbf A_{M,\Delta t}^{-1}\|
     \Big(
          |\mathbf b_{M,\Delta t}-\mathbf b_\Delta|
          + \|\mathbf A_{M,\Delta t}-\mathbf A_\Delta\|\,|\theta_\Delta|
     \Big) \\ 
     &\le \frac{2}{\mu_*} \left( \sqrt{\frac{2C_b}{M\eta}} + B_\theta\sqrt{\frac{2C_A}{M\eta}} \right) \le C_2 \frac{1}{\sqrt{M\eta}},
\end{aligned}
\]
with constant $C_2: =\frac{4\sqrt{2}}{\mu_*}\max\left\{\sqrt{C_b},\; \sqrt{C_A} (|\theta_*|+C_1\Delta_0)\right\}$.

\smallskip
Combining the two bounds with \eqref{eq:error_split_theta}, we obtain on $\Omega_{M,\eta}$
\[
     |\widehat  \theta_{M,\Delta t}-\theta_*|
     \le C_1\Delta t + C_2\frac{1}{\sqrt{M\eta}}. 
\]
Since $\mathbb P(\Omega_{M,\eta})\ge 1-\eta$, this proves \eqref{eq:parametric_error}.
\end{proof}

\begin{proof}[Proof of Corollary~\ref{cor:trapezoid_bound}]
We decompose the error as in~\eqref{eq:error_split_theta}. For the discretization error, Lemma~\ref{lem:trap_discretization} gives $\|\mathbf A_\Delta^{\mathrm{trap}}-\mathbf A_*\| \le L_A'(\Delta t)^2$ and $|\mathbf b_\Delta^{\mathrm{trap}}-\mathbf b_*| \le \frac{\sigma^2 L_b'}{2}(\Delta t)^2$. By the same perturbation argument as in Step~1 of Theorem~\ref{thm:error_bound},
\[
    |\theta_\Delta^{\mathrm{trap}}-\theta_*| \le C_1'(\Delta t)^2, \qquad C_1' = \frac{4}{3\mu_*}\Big(\frac{\sigma^2 L_b'}{2}+L_A'|\theta_*|\Big).
\]
The sampling error bound is identical to Step~2 of Theorem~\ref{thm:error_bound}, since the trapezoidal averages are still sample means $\frac1M\sum_m(\cdot)$ and the concentration bounds in Lemma~\ref{lem:concentration} apply. Then, Eq.~\eqref{eq:trap_error} follows by combining the two bounds. 
\end{proof}

\subsection*{Declaration of generative AI and AI-assisted technologies in the manuscript preparation process}
During the preparation of this work, the authors used AI agents for numerical computations and writing assistance. After using these tools, the authors reviewed and edited the content as needed and take full responsibility for the content of the published article.

% acknowledgments: thanks and funding
\section*{Acknowledgments}
FL was in part supported in part by the National Science Foundation under grants DMS-2238486 and DMS-2511283. FL would like to thank Quanjun Lang and Yuan Gao for helpful discussions on formulating the self-test loss function.

% Compact bibliography spacing
\let\OLDthebibliography\thebibliography
\renewcommand\thebibliography[1]{%
  \OLDthebibliography{#1}%
  \setlength{\parskip}{0pt}%
  \setlength{\itemsep}{2pt plus 0.5ex}%
}
\bibliographystyle{plain}
\bibliography{ref_weak_IPS_learning.bib}

\end{document}